\documentclass[11pt,a4paper]{article}

\usepackage{fancyhdr}
\usepackage{amsmath,mathtools}
\usepackage{bm}
\usepackage{esint}
\usepackage{amsfonts}
\usepackage{amsthm}
\usepackage{amssymb}
\usepackage{amsbsy}
\usepackage{verbatim}
\usepackage{graphicx}
\usepackage{url}
\usepackage{mathrsfs}
\usepackage[hmargin = 1in,vmargin=.75in]{geometry}
\usepackage{color}
\usepackage{amstext}
\usepackage{stmaryrd}
\usepackage{enumerate}
\usepackage{algpseudocode}
\usepackage{algorithm}
\usepackage{pifont}
\usepackage{prettyref}
\usepackage{subcaption}
\usepackage{hyperref}

\font\msbm=msbm10

\numberwithin{equation}{section}

\theoremstyle{plain}
\newtheorem{theorem}{Theorem}[section]
\newtheorem{lemma}[theorem]{Lemma}
\newtheorem{corollary}[theorem]{Corollary}

\newtheorem{proposition}[theorem]{Proposition}
\newtheorem{definition}{Definition}[section]

\def\mathbb#1{\hbox{\msbm{#1}}}

\newcommand{\bb}{\boldsymbol{b}}

\newcommand{\be}{\boldsymbol{e}}
\newcommand{\bh}{\boldsymbol{h}}

\newcommand{\bn}{\boldsymbol{n}}
\newcommand{\bp}{\boldsymbol{p}}

\newcommand{\br}{\boldsymbol{r}}
\newcommand{\bs}{\boldsymbol{s}}

\newcommand{\bv}{\boldsymbol{v}}
\newcommand{\bw}{\boldsymbol{w}}
\newcommand{\bx}{\boldsymbol{x}}
\newcommand{\by}{\boldsymbol{y}}
\newcommand{\bz}{\boldsymbol{z}}
\newcommand{\bone}{\boldsymbol{1}}

\newcommand{\BC}{\boldsymbol{C}}
\newcommand{\BD}{\boldsymbol{D}}

\newcommand{\BH}{\boldsymbol{H}}
\newcommand{\BI}{\boldsymbol{I}}
\newcommand{\BJ}{\boldsymbol{J}}

\newcommand{\BP}{\boldsymbol{P}}

\newcommand{\BR}{\boldsymbol{R}}

\newcommand{\BU}{\boldsymbol{U}}
\newcommand{\BV}{\boldsymbol{V}}
\newcommand{\BW}{\boldsymbol{W}}
\newcommand{\BX}{\boldsymbol{X}}
\newcommand{\BY}{\boldsymbol{Y}}
\newcommand{\BZ}{\boldsymbol{Z}}

\newcommand{\BPi}{\boldsymbol{\Pi}}

\newcommand{\BDelta}{\boldsymbol{\Delta}}

\newcommand{\btheta}{\boldsymbol{\theta}}

\newcommand{\BSigma}{\boldsymbol{\Sigma}}

\newcommand{\pa}{\partial}

\newcommand{\I}{\boldsymbol{I}}
\newcommand{\RR}{\mathbb{R}}
\newcommand{\lag}{\langle}
\newcommand{\rag}{\rangle}

\newcommand{\eps}{\varepsilon}

\newcommand*\diff{\mathop{}\!\mathrm{d}}
\DeclareMathOperator{\Tr}{Tr}

\DeclareMathOperator{\diag}{diag}

\DeclareMathOperator{\sign}{sign}

\DeclareMathOperator{\blkdiag}{blkdiag}
\DeclareMathOperator{\argmin}{argmin}

\begin{document}
\title{\bf Neural Collapse for Unconstrained Feature Model under Cross-entropy Loss with Imbalanced Data}
\author{Wanli Hong\thanks{Shanghai Frontiers Science Center of Artificial Intelligence and Deep Learning, New York University Shanghai, China. S.L. and W.H. is (partially) financially supported by the National Key R\&D Program of China, Project Number 2021YFA1002800, National Natural Science Foundation of China (NSFC) No.12001372, Shanghai Municipal Education Commission (SMEC) via Grant 0920000112, and NYU Shanghai Boost Fund. W.H. is also supported by NYU Shanghai Ph.D. fellowship.}~\thanks{Center for Data Science, New York University.},~ Shuyang Ling$^*$}
 
\date{\today}

\maketitle

\begin{abstract}
Recent years have witnessed the huge success of deep neural networks (DNNs) in various tasks of computer vision and text processing. Interestingly, these DNNs with massive number of parameters share similar structural properties on their feature representation and last-layer classifier at terminal phase of training (TPT). 
Specifically, if the training data are balanced (each class shares the same number of samples), it is observed that the feature vectors of samples from the same class converge to their corresponding in-class mean features and their pairwise angles are the same. This fascinating phenomenon is known as Neural Collapse $({\cal NC})$, first termed by Papyan, Han, and Donoho in 2019. 
Many recent works manage to theoretically explain this phenomenon by adopting so-called unconstrained feature model (UFM). In this paper, we study the extension of ${\cal NC}$ phenomenon to the imbalanced data under cross-entropy loss function in the context of unconstrained feature model.
Our contribution is multi-fold compared with the state-of-the-art results: (a) we show that the feature vectors exhibit collapse phenomenon, i.e., the features within the same class collapse to the same mean vector; (b) the mean feature vectors no longer form an equiangular tight frame. Instead, their pairwise angles depend on the sample size; (c) we also precisely characterize the sharp threshold on which the minority collapse (the feature vectors of the minority groups collapse to one single vector) will take place; (d) finally, we argue that the effect of the imbalance in datasize diminishes as the sample size grows.
Our results provide a complete picture of the ${\cal NC}$ under the cross-entropy loss for the imbalanced data. Numerical experiments confirm our theoretical analysis.
\end{abstract}

\section{Introduction}
Deep neural networks (DNNs) have achieved impressive results in various classification tasks~\cite{HZRS16,KSH12,LBH15,SZ14,SLJ15}. However, its highly nonconvex nature along with the massive number of parameters and distinct training paradigms pose great challenges for conducting theoretical analysis.
A recent thread of works studies the ways to keep optimizing the model in the terminal phase of training (TPT) when the training loss is very close to zero to achieve a better generalization power~\cite{HHS17,BRT19,BHMM19}. Therefore, theoretical studies about such over-parametrized neural networks in this regime become helpful in demystifying DNNs so as to design better training paradigms.

Neural collapse $({\cal NC})$ is a phenomenon observed in~\cite{PHD20} that some particular structures emerge in the feature representation layer and the classification layer of DNNs in the TPT regime for classification tasks when the training dataset is balanced. It has been also observed and studied under the mean-squared loss~\cite{HPD21,PL21a,ZLD22} and in many different settings~\cite{EP21,JLZDS21,TZ15}. For the simplicity of future discussion, we restate the four types of collapses introduced in~\cite{PHD20}:
\begin{itemize}
\item $\mathcal{NC}_1$: the feature of samples from the same class converge to a unique mean feature vector; 
\item $\mathcal{NC}_2$: these feature vectors (after centering by their global mean) form an equiangular tight frame (ETF), i.e., they share the same pairwise angles and length; 
\item $\mathcal{NC}_3$: the weight of the linear classifier converges to the corresponding feature mean (up to scalar product); 
\item $\mathcal{NC}_4$: the trained DNN classifies the sample by finding the closest mean feature vectors to the sample feature. 
\end{itemize}

After this empirical finding, many works follow to theoretically explain why $\mathcal{NC}$ occurs in DNNs. Staring from~\cite{EW22,FHLS21,LS22,MPP20}, a thread of works consider the unconstrained feature model (UFM) to simulate the process of training DNNs. The UFM simplifies a deep neural network into an optimization program by treating the features of training data as free variables to optimize over. Such simplification is based upon the rationale of universal approximation theorem~\cite{HSW89}: deep neural networks can well approximate a large variety of functions provided that the neural network is sufficiently over-parameterized.
Various versions of UFMs with different loss functions and regularizations are proposed in these works~\cite{EW22,MPP20,ZDZ21,FHLS21,DNT23,ZYL22,LS22,TB22,THN23,MPP20,YWZB22}. They all manage to find that the global minimizers of the empirical risk function under the UFMs match the characterization of ${\cal NC}$ proposed in~\cite{PHD20}.

While there are many recent works focusing on balanced datasets, we take a step further to see how $\mathcal{NC}$ generalizes to imbalanced datasets. Several works have already addressed phenomena in the imbalanced scenario. In particular,~\cite{FHLS21} found a phenomenon called~\emph{minority collapse} in the TPT regime for the training on imbalanced data. They empirically observed that under cross-entropy loss, as the imbalance ratio goes to infinity, the pairwise angles among minority classes become zero, which means the predictions on the minority classes become indistinguishable. It is believed that if the imbalance ratio is above a certain threshold, this minority collapse occurs but the exact threshold is unknown. A recent work~\cite{DNT23} obtained this exact threshold under the mean-square error (MSE) loss. 
The work~\cite{TKV22} considered the neural collapse for the imbalanced dataset under the unconstrained-feature SVM (UF-SVM) and proposed Simplex-Encoded-Labels Interpolation (SELI) geometry that characterized the structure of global minimizers to the UF-SVM, and later~\cite{BKV23} extended~\cite{TKV22} to several cross-entropy parameterizations.

However, to the best of our knowledge, the ${\cal NC}$ on imbalanced datasets under the cross-entropy loss is not fully understood. Our work will try to address a few questions that are not yet answered in the current literature: 
\begin{enumerate}[~~~(a)]
\item Does ${\cal NC}_1$ still occur for the UFMs under cross-entropy loss if the data are imbalanced?
\item If ${\cal NC}_1$ holds, what is the structure of the mean feature or prediction matrices? 
\item Can we provide a sharp threshold for the minority collapse? \item How does the imbalance ratio affect the structure of the mean feature vectors if the sample size is sufficiently large? 
\end{enumerate}

For (a) and (b), these questions are answered under the MSE loss in~\cite{DNT23}. However, it becomes much more challenging under the cross-entropy loss. While~\cite{FHLS21} studied a special case when there are two giant clusters, a clear characterization of the general case is unknown under the cross-entropy loss. For (c), the threshold for minority collapse remains unknown and we aim to fill this gap. For (d), we have not observed any recent works investigating this issue.

By adopting the UFMs under the cross-entropy loss, we provide a complete picture of ${\cal NC}$ for the imbalanced scenario. Here are our main contributions: (i) We provide a concise proof for $\mathcal{NC}_1$ under the cross-entropy loss for imbalanced datasets. This argument is flexible and can be easily applied to other UFM settings and different loss functions.
 Additionally, we find $\mathcal{NC}_2$ and $\mathcal{NC}_3$ do not hold for imbalanced datasets (Theorem~\ref{thm:main1}).
(ii) By working with imbalanced datasets where the classes are partitioned into clusters such that classes from the same cluster share the same number of samples, we show that the mean feature, prediction and classifier weight vector of the classes from the same cluster form an ETF-like structure. Moreover, we show the bias terms corresponding to the same cluster also share the same value (Theorem~\ref{thm:main1}).
(iii) When there are only two clusters (majority and minority cluster), which is the same setting adopted in~\cite{FHLS21}, we provide an exact threshold for the minority collapse. Moreover, we also characterize the threshold for complete collapse (Theorem~\ref{thm:biasf}), in which case all the classes collapse to a single vector. 
(iv) We provide an asymptotic characterization when the number of samples in the majority and minority cluster goes to infinity, but the imbalance ratio stays constant. We find $\mathcal{NC}_2$ and $\mathcal{NC}_3$ hold asymptotically and the convergence rate w.r.t. the sample size is provided (Theorem~\ref{thm:limit}).

\subsection{Notation}
We let boldface letter $\BX$ and $\bx$ be a matrix and a vector respectively; $\BX^{\top}$ and $\bx^{\top}$ are the transpose of $\BX$ and $\bx$ respectively. The matrices $\I_n$, $\BJ_n$, and $\be_k$ are the $n\times n$ identity matrix, a constant matrix with all entries equal to 1, and the one-hot vector with $k$-th entry equal to 1. For the simplicity of notation, we also let
\begin{equation}\label{def:Ck}
\BC_K := \I_K - \BJ_K/K
\end{equation}
be the $K\times K$ centering matrix. 
For any vector $\bx,$ $\diag(\bx)$ denotes the diagonal matrix whose diagonal entries  equal $\bx.$ 
For any matrix $\BX$, we let $\|\BX\|$, $\|\BX\|_F$, and $\|\BX\|_*$ be the operator norm, Frobenius form, and nuclear norm. 

\subsection{Organization}

The following sections are organized in the following way. In Section~\ref{s:prelim}, we formally introduce $\mathcal{NC}$ and our UFM along with a review of more recent works about $\mathcal{NC}$. In Section~\ref{s:main}, we present our main theoretical results. In Section~\ref{s:numerics}, we provide numerical experiments to support our theoretical findings and Section~\ref{s:proof} justifies all our theorems.

\section{Preliminaries}\label{s:prelim}
In this section, we briefly introduce DNNs, and then define the UFM and $\mathcal{NC}$ formally. 
A deep neural network (DNN) is often in the form of
\[
f_{\Theta}(\bx) = \BW^{\top} \bh_{\btheta}(\bx) + \bb
\]
where $\bh_{\btheta} \in \mathbb{R}^{d}$ represents the feature vector on the last layer, $\BW \in \mathbb{R}^{K \times d}$ and $\bb \in \mathbb{R}^{K}$ stand for the weight and bias respectively. The capital letter $\Theta$ consists of all the training parameters $(\btheta,\BW,\bb)$ in the DNN. In addition, we call $\bx$ the input and $f_{\Theta}(\bx)$ the  prediction vector of $\bx$. Given the training data $\{(\bx_i, \by_i)\}_{i=1}^N$, we try to find a model via empirical risk minimization (ERM):
\[
\min_{\Theta}~ \frac{1}{N} \sum_{i=1}^N \ell(f_{\btheta}(\bx_i), \by_i) + \frac{\lambda}{2}\|\Theta \|^2
\]
where $\ell(\cdot,\cdot)$ denotes a loss function,  $\by_i$ is a one-hot vector representing the label of the $i$-th training data $\bx_i$, and $\lambda>0$ is the regularization parameter (i.e., weight decay parameter of SGD). For the classification tasks, we will use the cross-entropy (CE) function $\ell_{CE}(\cdot,\cdot)$, i.e.,
\[
\ell_{CE}(\bz,\be_k) = \log \frac{\sum_{\ell=1}^K e^{z_{\ell}}}{e^{z_k}} = \log \sum_{\ell=1}^K e^{z_{\ell}} - z_k.
\]

We let $\bh_{ki} : = \bh_{\btheta}(\bx_{ki})$ be the feature of the $i$-th data point in the $k$-th class with $1\leq i\leq n_k$ and $1\leq k\leq K$, and $N = \sum_{k=1}^K n_k$ is the total number of samples. In other words, there are in total $K$ different classes with the $k$-th class containing $n_k$ samples. Without loss of generality, we let $\{n_k\}_{k=1}^K$ be a non-increasing sequence, i.e., $n_1\geq n_2\geq\cdots\geq n_K$. To simplify the notation, 
we let $\BH\in\mathbb{R}^{d \times N}$ be the feature matrix of all training samples with $\bh_{ki}$ denoting the $\left(\sum_{i=1}^{k-1} n_k + i\right)$-th column of $\BH$.
Now we are ready to introduce UFMs and $\mathcal{NC}$ related results.

\subsection{Unconstrained feature model}\label{ss:ufm}
For general DNNs, the feature $\bh_{\theta}(\cdot)$ is always highly nonlinear and thus challenging to analyze. The unconstrained feature model (UFM) simplifies the DNN model by assuming $\bh_{\theta}(\cdot)$ as a free vector, by using the idea that if a neural network is sufficiently parameterized, it can interpolate any data. Under the UFM, we instead study the regularized empirical risk minimization (ERM):
\[
\min_{\BW\in\RR^{d\times K},\BH\in\RR^{d\times N}} \frac{1}{N} \sum_{k=1}^K\sum_{i=1}^{n_k}\ell_{CE}( \BW^{\top}\bh_{ki} + \bb, \be_k )  + \frac{\lambda_W}{2}\|\BW\|_F^2 + \frac{\lambda_H}{2}\|\BH\|_F^2 + \frac{\lambda_b}{2}\|\bb\|^2
\]
where $(\lambda_W,\lambda_H,\lambda_b)$ are positive regularization parameters. It has an equivalent matrix form: 
\begin{equation}\label{eq:ufm}
\min_{\BW\in\RR^{d\times K},\BH\in\RR^{d\times N}} {\cal L}(\BW,\BH,\bb) : = \frac{1}{N} \ell_{CE}( \BW^{\top}\BH + \bb\bone_N^{\top}, \BY )  + \frac{\lambda_W}{2}\|\BW\|_F^2 + \frac{\lambda_H}{2}\|\BH\|_F^2 + \frac{\lambda_b}{2}\|\bb\|^2
\end{equation}
where $\BW\in\RR^{d\times K},$ $\BH= \{\bh_{ki}\}_{1\leq i\leq n_k,1\leq k\leq K}\in\RR^{d\times N}$,  
\begin{equation}\label{def:Y}
\BY = [ \be_1\bone_{n_1}^{\top},\cdots,\be_K\bone_{n_K}^{\top} ]\in\RR^{K\times N},
\end{equation}
 and $\ell_{CE}(\BW^{\top}\BH+\bb\bone_N^{\top},\BY)$ computes the cross entropy column-wisely and then takes the sum.

It is a great convenience to work with model~\eqref{eq:ufm} as we can convexify the problem. Under $d\geq K$, i.e., in the regime of over-parameterization, then $\BW^{\top}\BH$ can represent any $K\times N$ matrix. Therefore, let $\BZ = \BW^{\top}\BH\in\RR^{K\times N}$ and we have
\begin{equation}\label{eq:nuc}
\min_{\BW^{\top}\BH =\BZ} \lambda_W\|\BW\|_F^2 +\lambda_H \|\BH\|_F^2 = 2\sqrt{\lambda_W\lambda_H}\|\BZ\|_*
\end{equation}
which follows from~\cite[Lemma 5.1]{RFP10} and~\cite[Lemma A.3]{ZDZ21}.
By letting $\lambda_Z := \sqrt{\lambda_W\lambda_H}$,~\eqref{eq:ufm} becomes
\begin{equation}\label{eq:cvx}
  \tag{\bf UFM}
\min_{\BZ\in\RR^{K\times N},\bb\in\RR^K} {\cal L}(\BZ,\bb) : = \frac{1}{N}\ell_{CE}( \BZ + \bb\bone_N^{\top}, \BY )  + \lambda_Z \|\BZ\|_* + \frac{\lambda_b}{2}\|\bb\|^2
\end{equation}
which is a convex optimization problem. 

Therefore, it suffices to focus on the structure of global minimizers to~\eqref{eq:cvx}, as it implies the global minimizer to~\eqref{eq:ufm}. To see this, let $\BZ^*$ be a global minimizer to~\eqref{eq:cvx} and $\BZ^* = \BU\BSigma\BV^{\top}$ be its SVD. Then $\BW\propto \BSigma^{1/2}\BU^{\top}$ and $\BH \propto \BSigma^{1/2}\BV^{\top}$ are actually the global minimizer to~\eqref{eq:ufm}, which follows from~\cite[Lemma 5.1]{RFP10}.
As a result, our focus will be on analyzing $\BZ$ instead of its factorized form $\BZ = \BW^{\top}\BH$.  In particular, we will call $\BZ$ the~\emph{prediction matrix}.

\subsection{Neural collapse}
In this section, we will review more recent works relevant to ours. 
Regarding the theoretical understanding of ${\cal NC}$, the research on the UFMs has become popular in the past few years. Besides the UFM we have introduced above, there are several variants of the UFMs in the state-of-the-art literature. Most works focus on characterizing the global solution of the corresponding regularized empirical risk function and aim to show that it captures $\mathcal{NC}$ phenomenon on the balanced dataset. For more details, we refer readers to works such as~\cite{KRA22,ZDZ21} and the references therein.
With the notation introduced in Section~\ref{ss:ufm} at hand, we can describe the ${\cal NC}$ more precisely.
\begin{itemize}
  \item $\mathcal{NC}_1$ - within-class variability collapse: $\bh_{ki} = \bar{\bh}_k$ for $1\leq i\leq n_k$ and $1\leq k\leq K$; 
  \item $\mathcal{NC}_2$ - convergence of the mean features to an ETF. Let $\bar{\BH}=[\bar{\bh}_1,\bar{\bh}_2,\cdots, \bar{\bh}_k] \in \mathbb{R}^{d \times K}$ be the mean feature matrix. Then it holds $\bar{\BH}^{\top}\bar{\BH} \propto \BC_K$, i.e., the mean features form an equiangular tight frame (a regular simplex);
  \item $\mathcal{NC}_3$ - self-duality. The weight matrix $\BW$ is proportional to $\bar{\BH}^{\top}$.
\end{itemize}

\paragraph{${\cal NC}$ on balanced datasets:} The ${\cal NC}$ is first empirically observed on the balanced dataset in~\cite{PHD20}. Hence most follow-up works focus on the balanced scenario, i.e., $n_1=\cdots = n_K$. The goal is to prove the global minimizer associated with the ERM satisfies ${\cal NC}_1$-${\cal NC}_3$ in~\cite{PHD20}  under certain UFMs.  
The work~\cite{FHLS21} studies the neural collapse under the bias-free unconstrained feature model (termed as the layer-peeled model in~\cite{FHLS21}), and shows ${\cal NC}_1$-${\cal NC}_3$ hold in the balanced scenario with an $\ell_2$-norm constraint on $(\BW,\BH,\bb)$. Several works have provided similar results such as~\cite{EW22,LS22,ZDZ21}. In particular, the authors in~\cite{ZDZ21} characterize the benign landscape of the regularized ERM by showing that there is only one local minimizer that is also global, modulo a global rotation.

The ${\cal NC}$ under the UFMs with MSE loss has also been studied in~\cite{DNT23,HPD21,TB22,ZLD22}. While the within-class collapse ${\cal NC}_1$ still holds, $\mathcal{NC}_2$ exhibits a slightly different structure: the mean feature vectors in $\bar{\BH}$ become pairwise orthogonal, i.e., $\bar{\BH}^{\top}\bar{\BH} \propto \I_K$. This is due to the difference between the CE  and MSE loss. 
Other loss functions including loss label smoothing and focal loss have been considered in~\cite{ZYL22} to demonstrate the universality of $\mathcal{NC}$. 

\paragraph{${\cal NC}$ on imbalanced data:} 
The work~\cite{FHLS21} is likely the first to consider the $\mathcal{NC}$ for the imbalanced data under the UFM and cross-entropy loss. They work with a dataset consisting of two giant clusters $A$ and $B$: each cluster $A$ (or $B$)  contains $k_A$ (or $k_B$) classes and each class contains $n_A$ (or $n_B$) samples, i.e., $n_A := n_1=n_2=\cdots =n_{k_A}$ and $n_B:=n_{k_A+1}=n_{k_A+2}\cdots =n_{k_A+k_B}$. Without loss of generality, we assume $n_A>n_B$, and $A$ and $B$ are referred to as majority and minority class respectively.
In~\cite{FHLS21}, it is empirically observed that the ${\cal NC}_1$ occurs. Moreover, when the imbalance ratio $r:=n_A/n_B$ is greater than some threshold, all the mean feature vectors w.r.t. the minority class become the same, which means the prediction on the classes in $B$ becomes indistinguishable. This phenomenon is termed as the {\em minority collapse}. Theoretically,~\cite{FHLS21} shows minority collapse when the imbalance ratio $r$ is sufficiently large but the exact threshold remains unknown. 

For the minority collapse under MSE loss,~\cite{DNT23} has explicitly characterized the collapse threshold for each class in terms of the regularization parameters and number of samples. The argument essentially follows from the idea of singular value thresholding~\cite{CCS10}. In particular,~\cite{BKV23} provides the pairwise angle within the minority and majority classes under two parameterizations of the CE loss.

\paragraph{${\cal NC}$ beyond the UFMs:}
There are a few other works concerning slightly more complicated models beyond the UFMs. Recently,~\cite{YWZB22} has taken one step forward from the UFMs by restricting the weight $\bw_k$ and feature $\bh_{ki}$ on the unit ball, also known as the normalized features, and has analyzed the ${\cal NC}$ under this restricted setting. One disadvantage of the UFM is that the model ignores the network depth and nonlinearity, and also the dependence of the feature vector on the input sample. A few progress in this direction include~\cite{DNT23} which considers deep linear networks and explores the $\mathcal{NC}$ under the MSE. 
In addition,~\cite{TB22} adds a bit of nonlinearity by applying the ReLU activation to the features $\BH$ before feeding to the linear classifier. Recently,~\cite{SWGKC23} has explored the connection between the neural collapse and neural tangent kernel~\cite{JGH18}.

\paragraph{${\cal NC}$ and training/generalization} Now we briefly review a few other works that are relevant to the ${\cal NC}$.
Regarding the stability of ${\cal NC}$, the work~\cite{THN23} considers initializing the input feature near the collapsed solution and conducts perturbation analysis in the near-collapse regime. Motivated by the ETF type mean feature vectors,~\cite{YXC22,ZDZ21} consider training with the last layer fixed as an ETF; this training scheme achieves on-par performance compared with that with the classifier not fixed. This may be used as a potential way to decrease the computational costs of training DNNs. The works~\cite{GGH21,GGH22} show that few-shot learning achieves good performance by adopting transfer learning on a trained-to-collapsed network except for the last classifier layer. A similar setting of transfer learning is also considered in~\cite{LLZ22}.

Another important aspect is the connection between ${\cal NC}$ and generalization~\cite{ESA20,HBN22}.
The recent work~\cite{HBN22} has examined their relation empirically. They find the collapse on the testing dataset does not take place on benchmark datasets including MNIST, FMNIST and Cifar10. They point out that  $\mathcal{NC}$ is not desirable in certain transfer learning settings. Additionally,~\cite{HBN22} also observes the ${\cal NC}$ starts to occur on a few layers before the last layer, known as the cascading collapse.

\section{Main results}\label{s:main}

The global minimizer to~\eqref{eq:cvx} in the balanced scenario forms exactly an equiangular tight frame~\cite{EW22,FHLS21,ZDZ21}. However, it is unclear how this phenomenon is affected by the number of samples in each class. In this section, we will present our findings on neural collapse under the imbalanced scenario.
Before proceeding to our main results, we need to introduce the cluster structure.  
\begin{definition}[\bf Cluster structure]
Let $\{N_j\}_{j=1}^{J}$  be the distinct values of $\{n_k\}_{k=1}^K$ with $J\leq K$ and 
\begin{equation}\label{def:gamma}
\Gamma_j = \{k: n_k = N_j,~1\leq k\leq K\}
\end{equation}
We call $\Gamma_j$ the $j$-th cluster, i.e., every class in $\Gamma_j$ has $N_j$ samples.
\end{definition}

Now we present the first main theorem, regarding the structure of the global minimizer to~\eqref{eq:cvx}.
\begin{theorem}\label{thm:main1}

The  global minimizer $(\BZ,\bb)$ to~\eqref{eq:cvx} is unique and satisfies the following properties:
\begin{enumerate}[(a)]
\item {\bf (Within-class feature collapse)} The ${\cal NC}_1$ occurs for unconstrained feature models under cross-entropy loss: the prediction vectors $\bz_{ki},~1\leq i\leq n_k$ within each class collapse to their sample mean $\bar{\bz}_k$:
\[
\bz_{ki} = \bar{\bz}_k,~~1\leq i\leq n_k,~~~\lag \bar{\bz}_k, \bone_K\rag = 0,~~1\leq k\leq K,
\] 
In other words, the prediction matrix $\BZ$ is in the following factorized form:
\[
\BZ = \bar{\BZ} \BY= \in\RR^{K\times N}
\]
where
\begin{equation}\label{def:barZS}
\bar{\BZ} = [\bar{\bz}_1,\cdots,\bar{\bz}_K],~~~\BY\text{ is defined in}~\eqref{def:Y}.
\end{equation}
From now on, we refer to $\bar{\BZ}$ as the mean prediction matrix.

\item {\bf (Block structure of $\bar{\BZ}$)} 
The mean prediction matrix $\bar{\BZ}$ and the bias term $\bb$ exhibit the block structure:
\[
\bar{\BZ} = \sum_{j=1}^J a_j \I_{\Gamma_j} + \sum_{1\leq j,j'\leq J} a_{jj'} \bone_{\Gamma_j}\bone_{\Gamma_{j'}}^{\top},~~~\bb = \sum_{j=1}^J c_j \bone_{\Gamma_j},~~~a_j + \sum_{j'=1}^J a_{j'j} |\Gamma_{j'}| = 0,
\]
where $\bone_{\Gamma_j}$ is an indicator vector, defined by
\[
\bone_{\Gamma_j}(\ell) = 
\begin{cases}
1, & \ell\in\Gamma_j \\
0, & \ell\in\Gamma_j^c
\end{cases},\qquad \I_{\Gamma_j} = \diag(\bone_{\Gamma_j}).
\]
In other words, the mean prediction vectors $\{\bar{\bz}_k\}_{k\in N_j}$ in the same cluster have the same pairwise angle, so do the mean feature matrix $\bar{\BH}.$ 

\item {\bf (Balanced scenario as a special case)}
If $n_1=n_2=\cdots=n_K = N/K$, then
\[
\bar{\BZ} = a\left(K\I_K - \BJ_K\right),~~\bb = 0.
\]
In particular, we have
\begin{enumerate}[i.]
\item if $N\lambda_Z \geq \sqrt{\frac{N}{K}}$, then $a = 0$;
\item if $N\lambda_Z < \sqrt{\frac{N}{K}}$, then 
\[
a = \frac{1}{K}\log \left( \frac{\sqrt{K}}{\sqrt{N}\lambda_Z} - K + 1 \right).
\]
\end{enumerate}

\item The weight $\BW$ and feature matrix $\BH$ also have a block structure. More precisely, let $\bar{\BU}\bar{\BSigma}\bar{\BV}^{\top}$ be the SVD of $\bar{\BZ}\BD^{1/2}$ where
\begin{equation}\label{def:D}
\BD := \BY\BY^{\top}= \diag(n_1,\cdots,n_K).
\end{equation}
Then $\BH = \bar{\BH}\BY$ and the mean prediction $\bar{\BZ}$ equals $\BW^{\top}\bar{\BH}$ where
\[
\BW=\bar{\BSigma}^{1/2}\bar{\BU}^{\top},~~~\bar{\BH} = \bar{\BSigma}^{1/2}\bar{\BV}^{\top}\BD^{-1/2}.
\]

\end{enumerate}
\end{theorem}

The theorem above indicates the block structure of the global minimizer to~\eqref{eq:cvx}. For a numerical illustration, we refer the readers to Figure~\ref{FIG:Block Structure} in our numerical section.
In particular, Theorem~\ref{thm:main1}(c) exactly recovers the existing results on the neural collapse for balanced datasets in~\cite{FHLS21,ZDZ21}. It is worth pointing out that our proof technique is much more general than those in~\cite{FHLS21,ZDZ21}, and a similar argument also applies to the normalized features~\cite{YWZB22}.

\vskip0.25cm

Despite Theorem~\ref{thm:main1} characterizes the structure of global minimizers, it does not give insights into how the sample size in each class affects the global minimizers. 
Next, we focus on a special case where there are two giant clusters, denoted by $A$ and $B$. In the cluster $A$ (or $B$), there are $k_A$ ( or $k_B$) classes with the sample size of each individual class equal to $n_A$ (or $n_B$). Without loss of generality, we let $n_A > n_B$ and refer $A$ ($B$) as the majority (minority) class.  Hence $N = k_An_A+k_Bn_B$ and $K = k_A+k_B$. 

Note that Theorem~\ref{thm:main1}(a) implies that the global minimizer of $\BZ$ and $\bb$ exhibit block structures.
Therefore, we will frequently use the following $2\times 2$ block matrix. We say a matrix $\BX\in\RR^{K\times K}$ equals ${\cal B}(a_X,b_X,c_X,d_X)$ if
\begin{equation}\label{def:calB}
\BX = \begin{bmatrix}
  a_X(k_A\I_{k_A} -\BJ_{k_A\times k_A}) + c_Xk_B \I_{k_A}  & -b_X\BJ_{k_A\times k_B} \\
  -c_X\BJ_{k_B\times k_A} & d_X(k_B\I_{k_B}  -\BJ_{k_B\times k_B}) + b_Xk_A\I_{k_B}
  \end{bmatrix}\in\RR^{K\times K}
\end{equation}

Without loss of generality, we assume $\bar{\BZ}$ (the within-class mean of $\BZ$) and $\bb$ are
 \begin{equation}\label{eq:barZb}
\begin{aligned}
\bar{\BZ} = {\cal B}(a,b,c,d)\in\RR^{K\times K}, \quad
  \bb  = m
  \begin{bmatrix}
  k_B\bone_{k_A}\\
  -k_A\bone_{k_B}
  \end{bmatrix} \in \mathbb{R}^K
\end{aligned}
\end{equation}
for some parameter $a,b,c,d$ and $m.$
The next theorem provides a detailed characterization of how the solution structure of $\bar{\BZ}$ depends on $\lambda_Z.$ This theorem will be crucial in characterizing the threshold for minority collapse.

\begin{theorem}[{\bf Block structure v.s. $\lambda_Z$}]\label{thm:biasf}
Assume $n_A>n_B$, and $N = k_An_A + k_Bn_B$ with $k_A\geq 2$ and $k_B\geq 2.$
For $\lambda_Z$ of different regimes, the optimal solution is $\BZ=\bar{\BZ}\BY$ with the mean prediction matrix $\bar{\BZ}$ in the form of~\eqref{eq:barZb}.
\begin{enumerate}[(a)]
\item If $N\lambda_Z \leq \min\{\sqrt{n_A},\sqrt{n_B}\}$, $\bar{\BZ}$ is unique in the form of~\eqref{eq:barZb} that satisfies $a,b,c,d>0$ and
\[
a-c+d-b \leq 0.
\]

\item If $\sqrt{n_B} < N\lambda_Z < \sqrt{n_A}$ and $\xi(\lambda_Z,\lambda_b) < 0$ for some nonlinear function $\xi$ in~\eqref{def:xi}, then
\[
  \bar{\BZ} 
  = \begin{bmatrix}
  a(k_A\I_{k_A} -\BJ_{k_A\times k_A}) + ck_B \I_{k_A}  & -b\BJ_{k_A\times k_B} \\
  -c\BJ_{k_B\times k_A} &  \frac{k_A}{k_B}b \BJ_{k_B\times k_B} 
  \end{bmatrix}\in\RR^{K\times K}.
\]
Moreover, $(a,b,c,d)$ satisfies $b>0$, and $c>0$. In particular, $\exists\eps > 0$ such that for any $\lambda_Z \in [\sqrt{n_B}/N, \sqrt{n_B}/N +\eps]$, $\xi(\lambda_Z,\lambda_b) < 0$ holds.

\item If $\sqrt{n_B} < N\lambda_Z < \sqrt{n_A}$ and $\xi(\lambda_Z,\lambda_b) > 0$ for some nonlinear function $\xi$ in~\eqref{def:xi}, then
\[
  \bar{\BZ} 
  = \begin{bmatrix}
  a(k_A\I_{k_A} -\BJ_{k_A\times k_A})  & 0 \\
  0 &  0 
  \end{bmatrix}\in\RR^{K\times K}.
\]
Moreover, $(a,b,c,d)$ satisfies $b=c=d=0$. In particular, $\exists\eps > 0$ such that for any $\lambda_Z \in [\sqrt{n_A}/N-\eps, \sqrt{n_A}/N]$, $\xi(\lambda_Z,\lambda_b) > 0$ holds.

\item If $N\lambda_Z > \max\{\sqrt{n_A},\sqrt{n_B}\}$, then $\bar{\BZ} = 0$.

\item In particular, for the bias-free scenario, i.e., $\lambda_b = \infty$, then $\xi(\lambda_Z,\infty) < 0 $ $(>0)$ is equivalent to $\sqrt{n_B}/N < \lambda_Z < \lambda^*$ $(\lambda^* < \lambda_Z<\sqrt{n_A}/N)$ respectively for some $\lambda^*\in (\sqrt{n_B}/N,\sqrt{n_A}/N).$ The threshold $\lambda^*$ is the unique solution to a nonlinear equation.

\end{enumerate}
\end{theorem}

In Theorem~\ref{thm:biasf}, we notice that the threshold for cases (b) and (c) in the bias-free scenario is simpler. This is because it is challenging to prove the monotonicity of the nonlinear function $\xi(\lambda_Z,\lambda_b)$ in $\lambda_Z$ for any fixed $\lambda_b>0$ where $\xi(\cdot,\cdot)$ is defined in~\eqref{def:xi}. However, for any $\lambda_b>0$, we are able to show that when $\lambda_Z$ is close to $\sqrt{n_B}/N$ (or $\sqrt{n_A}/N$), the corresponding $\xi$ satisfies $\xi<0$ $(\xi>0)$. But a clear characterization of how $\bar{\BZ}$ transits from case (b) to (c) is unavailable now. 
However, it is certain that for $\sqrt{n_B} < N\lambda_Z < \sqrt{n_A}$, the solution is either in case (b) or (c); moreover, the minority collapse occurs as long as $\lambda_Z > \sqrt{n_B}/N$ for the unconstrained feature model, i.e., the mean prediction $\bar{\BZ}$ on the minority group becomes a single vector, as we can see the right block of $\bar{\BZ}$ is rank-1 in both case (b) and (c).


The theorem above immediately leads to the following corollary which characterizes the sharp threshold on the minority collapse.
It is empirically observed by \cite{FHLS21} that the mean prediction of minority classes collapse to one vector when fixing $\lambda_Z$ and $\lambda_b$ as the imbalance ratio $r=n_A/n_B$ increases and is greater than some threshold. Based on Theorem~\ref{thm:biasf}, we are able to give an explicit characterization of this critical threshold. 

\begin{corollary}[\bf Minority collapse threshold ($r\rightarrow \infty$, $n_B$ is fixed)]\label{corollary:mc}
  Suppose $r= n_A/n_B>1$, and $\lambda_Z,k_A,k_B$, and $n_B$ are fixed. Assume  
 \[
 r =\frac{n_A}{n_B} \geq \frac{1}{k_A} \left[\frac{1}{\lambda_Z\sqrt{n_B}}-k_B\right],
\] 
the mean prediction matrix on minority classes collapses to one vector.
\end{corollary}

We proceed to provide some asymptotic characterization of the mean prediction $\bar{\BZ}$, when $n_A$ and $n_B$ go to infinity but their ratio stays constant.

\begin{theorem}[\bf $r = n_A/n_B > 1$ is fixed, $n_B\rightarrow \infty$]\label{thm:limit}
Assume
\begin{align*}
& N\lambda_Z = \lambda < \sqrt{n_B}~~ \text{is constant}, \\
& r = n_A/n_B~~\text{is constant}, \\
& N = k_A n_A + k_Bn_B, \text{ with } (k_A,k_B)~\text{fixed},\\
& \lambda_b^{-1} = o(\sqrt{N}\log{N}),
\end{align*}
then the global minimizer $\bar{\BZ}$ of the form $(a_N^*,b_N^*,c_N^*,d_N^*)$ in \eqref{eq:barZb} satisfies
\[
\lim_{N\rightarrow\infty} \max\left\{ \left| \frac{b_N^*}{c_N^*}-1\right|, \left| \frac{a_N^*}{c_N^*}-1\right|, \left| \frac{d_N^*}{b_N^*}-1\right|\right\} = O\left(\frac{1}{\log N}\right).
\]
In other words, the columns of $\bar{\BZ}$ converge to the ETF as $N \to +\infty$ with $\lambda$ and $r = n_A/n_B$ fixed, and so do the corresponding weight $\BW^{\top}$ and the mean feature matrix $\bar{\BH}^{\top}.$
\end{theorem}

As the original model~\eqref{eq:ufm} is non-convex, a natural concern is about the landscape of the original programming. Theorem 3.2 in~\cite{ZDZ21} proves a benign optimization landscape for~\eqref{eq:ufm} in the balanced scenario: all the critical points are either global minimum of~\eqref{eq:ufm} (also critical points of~\eqref{eq:cvx}) or saddle points. This result has been extended in~\cite{ZYL22} to characterize the optimization landscape for other loss functions.
Regarding the imbalanced scenario under the UFM and cross-entropy loss,  the optimization landscape of~\eqref{eq:ufm} is also benign.
\begin{theorem}\label{Thm:LA}
  Assume the feature dimension $d>K$, then $\mathcal{L}(\BW,\BH,\bb)$ in~\eqref{eq:ufm} is a strict saddle function with 
  no spurious local minimum, in the sense that
  \begin{itemize}
    \item Any local minimizer of~\eqref{eq:ufm} is a global minimizer.
    \item Any critical point $(\BW,\BH,\bb)$ that is not a local minimizer is a strict saddle point with negative curvature, i.e. the Hessian $\nabla^2 \mathcal{L}(\BW,\BH,\bb)$, at this critical point, is non-degenerate and has at least one negative eigenvalue.
  \end{itemize}
\end{theorem}

Theorem~\ref{Thm:LA} is a direct generalization of Theorem 3.2 in \cite{ZDZ21} from the balanced case to the imbalanced case. The proof (see Section C.1 in~\cite{ZDZ21}) also directly applies without any changes, and thus we do not repeat the proof here. For the completeness of the presentation, we briefly discuss the proof idea, which is to classify the critical points of~\eqref{eq:ufm} into two categories. We denote the cross-entropy loss part of~\eqref{eq:ufm} as $\mathcal{L}_1(\BX) =N^{-1} \ell_{CE}(\BX, \BY ).$
For any critical points $(\BW,\BH,\bb)$ of~\eqref{eq:ufm}, if
\begin{itemize}
\item $\left\|\nabla \mathcal{L}_1(\BW^{\top}\BH+\bb\bone_N^{\top}) \right\| \leq \sqrt{\lambda_{W}\lambda_{H}}:$ one can show these points are also critical points of the convexified programming \eqref{eq:cvx}, and thus the global minimum. Intuitively, the inequality constrains the norm of the gradient, so it becomes a legal subgradient of the nuclear norm.

\item $\left\|\nabla \mathcal{L}_1(\BW^{\top}\BH+\bb\bone_N^{\top}) \right\| > \sqrt{\lambda_{W}\lambda_{H}}:$ One can construct a negative curvature direction in the null space of $\BW$, which is nonempty since $d>K$, and the singular vector corresponding to the largest singular value of $\nabla^2 \mathcal{L}_1(\BW,\BH,\bb)$.
\end{itemize}

We conclude this section by discussing our results and pointing out a few possible future directions. In conclusion, we present a rigorous and in-depth study into the neural collapse phenomenon for imbalanced dataset using the UFM with cross-entropy loss. In particular, we give a full characterization of the solution when the dataset has two clusters under the UFM, thus precisely finding the minority collapse threshold. This sharp threshold in Corollary~\ref{corollary:mc} is also confirmed in real experiments. 
As a result, one can select a suitable oversampling rate of minority classes to avoid minority collapse while saving computational resources and also not impairing test performance due to the high oversampling rate.

As later shown in Section~\ref{s:numerics}, our theory can only partially explain the behavior of real deep neural networks. For example, we can see non-negligible difference arises in certain regime between the predicted solution by the UFM and the actual prediction by the DNNs. Despite the landscape of UFM is benign by Theorem~\ref{Thm:LA}, the landscape of DNN is inherently different. This calls for more complicated models to explain DNNs. Several works~\cite{TB22,DNT23} try to add more linear layers to UFM, but adding even one layer of nonlinearity remains unexplored, which could be a future direction. It will be also interesting to find a weaker substitute of the UFM that interpolates between the DNNs and UFM. Finally, our paper does not discuss the relation between neural collapse and generalization error. To the best of our knowledge, most studies on that topic remain empirical. Addressing this issue theoretically will lead to a deeper understanding of the interplay between generalization and implicit bias broadly. We will leave these possible directions for future work. 

\section{Numerics}\label{s:numerics}

In this section, we present numerical results that confirm our theory and also give new insights. Our code is available on \href{https://github.com/WanliHongC/Neural-Collapse}{Github} and is adapted from the code by~\cite{FHLS21}. In the next four subsections, we show (a) the neural collapse phenomena arising from the imbalanced dataset; (b) the block structure of the mean prediction matrix $\bar{\BZ}$, and the difference between the feature mean and the exact solution obtained from solving~\eqref{eq:ufm}; (c) the sharp threshold of the minority collapse; and (d) the asymptotic behavior of $\bar{\BZ}$ as the sample size grows to infinity with fixed imbalance ratio.

We first briefly describe the training details.
To make the experiments and settings in~\eqref{eq:ufm} consistent, we place activation regularization on the last layer before the classification layer to model the regularization on features for every network and dataset we have trained.
All the networks, if not specified, are trained with a diminishing stepsize, as adopted in \cite{FHLS21}: the initial learning rate is 0.1 for the first 1/6 epochs; and after the first 1/6 epochs, we divide the learning rate by 10, i.e., learning rate equals 0.01, and train for another 1/6 epochs. After that, we set the learning rate as $10^{-3}$ for the rest of the training process.
All the networks are trained by SGD with momentum 0.9, and batch size 128. Additionally, except for networks trained in subsection~\ref{ss:num_nc3} where we turn off the weight decay of SGD, we set the weight decay of SGD to be 5e-4 during training. Since we are comparing the minority collapse threshold against regularization parameters in~\ref{ss:num_nc3}, we turn the weight decay off to make the regularization effect exact.

In Section~\ref{ss:num_nc1} and~\ref{ss:num_nc2}, we train three neural network including VGG11, VGG13, and ResNet18 on Fashion MNIST (FMNIST) and Cifar10. To validate our theory under different imbalance levels, we pick the following two choices of parameters, denoted by Dataset$_1$ and Dataset$_2$.
The whole dataset (either FMNIST or Cifar 10) contains three giant clusters $A,B$, and $C$. For each group (e.g. $A$), it contains $k_A$ classes and each class contains $n_A$ samples. In other words, there are in total $k_A+k_B+k_C$ classes and $k_An_A + k_Bn_B+k_Cn_C$ samples. The sample size of each class in the same cluster is the same.
\begin{enumerate}
    \item Dataset$_1$: $k_A=4,k_B=k_C=3,n_A=5000,n_B=4000,n_C=3000$
    \item Dataset$_2$: $k_A=k_C=4,k_B=2,n_A=5000,n_B=3000,n_C=1000$
\end{enumerate}
From the settings above, we can see Dataset$_2$ is more imbalanced than Dataset$_1$.
For the regularization parameters, we use $\lambda_W = 10^{-3}, \lambda_H = 10^{-6}, \lambda_b = 10^{-2}$ and train each model for 1000 epochs. 

In Section~\ref{ss:num_nc3} and~\ref{ss:num_nc4}, the experiments are used to verify Theorem~\ref{thm:biasf}, and Theorem~\ref{thm:limit} and Corollary~\ref{corollary:mc} respectively. Therefore, we only consider two giant clusters $A$ and $B$, with the number of classes and within-class sample size equal to $(k_A,n_A)$ and $(k_B,n_B)$ respectively. The regularization parameters $\lambda_W$ and $\lambda_H$ are set as $\lambda_W = 10\lambda_Z$ and $\lambda_H = \lambda_Z/10$ for each given $\lambda_Z.$
We will provide the specific parameter settings in each section. The network is ResNet18 and it is trained on Cifar10 by using SGD for $2000$ epochs.

\subsection{Collapse of feature and prediction vectors}\label{ss:num_nc1}

We first show the collapse of within-class feature vectors, as predicted by Theorem~\ref{thm:main1}(a). To quantify the level of within-class collapse, we compute the within-class and between-class covariance:  
\[
\BSigma_W := \frac{1}{N} \sum_{k=1}^K\sum_{i=1}^{n_k} (\bh_{ki}-\bar{\bh}_k) (\bh_{ki}-\bar{\bh}_k)^{\top},~~ \BSigma_{B} := \frac{1}{K}\sum_{k=1}^K (\bar{\bh}_k-\bh_G)(\bar{\bh}_k-\bh_G)^{\top}
\]
where 
\[
\bh_G := \frac{1}{N}\sum_{k=1}^{K} \sum_{i=1}^{n_k} \bh_{ki},\quad \bar{\bh}_k := \frac{1}{n_k}\sum_{i=1}^{n_k} \bh_{ki},~~1\leq k\leq K,
\]
are the total and within-class means respectively.
The level of within-class collapse is measured by
\begin{equation}\label{eq:sampleconvergencemetric}
\mathcal{NC}_1 := \frac{1}{K} \Tr\left({\BSigma_W\BSigma_B^{\dagger}}\right).
\end{equation}
It is easy to see that if the within-class collapse occurs, ${\cal NC}_1$ should be very small, as $\BSigma_W$ is close to 0.

Figure~\ref{fig:collapse} plots the change of $\log {\cal NC}_1$ against the epochs.  We can see after training 1000 epochs, $\log {\cal NC}_1$ is near $-10$ across all networks and datasets except VGG11 on the Cifar10 dataset which attains $\log {\cal NC}_1 \approx -6$.  This is strong evidence of the within-class collapse, and it confirms our Theorem~\ref{thm:main1}.

\begin{figure}[!htb]
    \centering
       \includegraphics[width=0.9\linewidth]{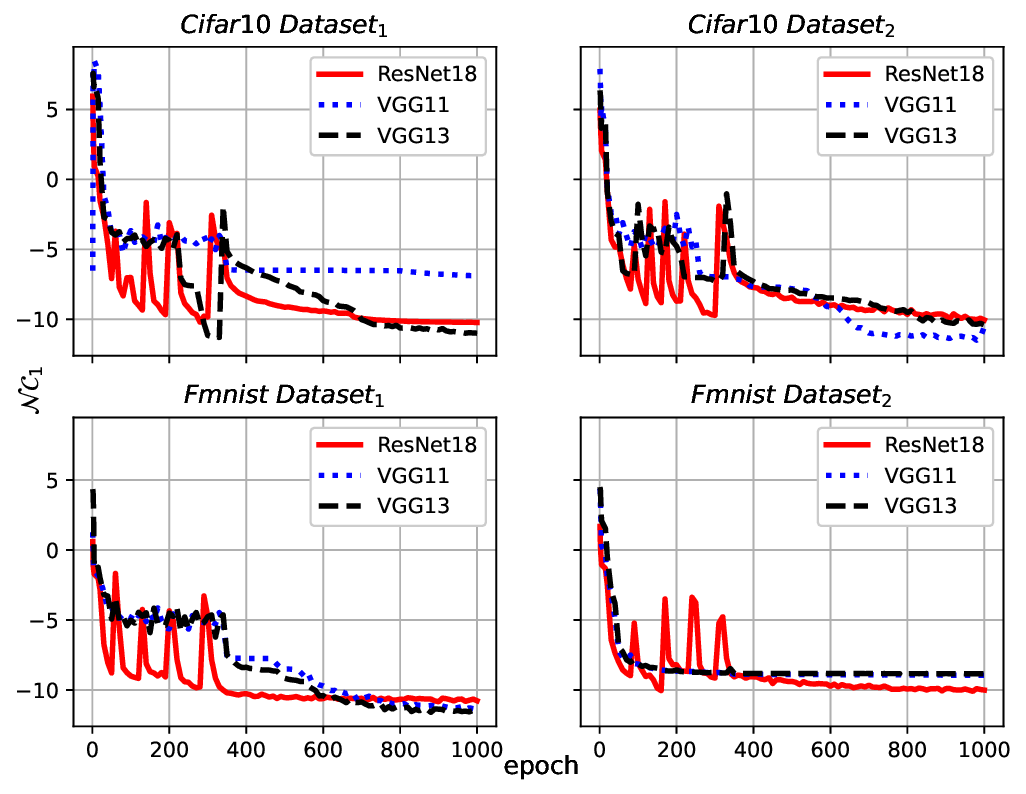}
    \caption{Plot of $\log {\cal NC}_1$ v.s. epochs: $x$-axis is the epoch number and y-axis is $\log {\cal NC}_1$. Datasets: Cifar10 and FMNIST with two sets of parameters Dataset$_1$ and Dataset$_2$. Network:  ResNet18 (red straight line), VGG11 (blue dotted line), and VGG13 (black dashed line).}
    \label{fig:collapse}
\end{figure}

\subsection{Block structure of the mean prediction and features}\label{ss:num_nc2}
In this subsection, we will illustrate the block structure of the mean prediction matrix to confirm Theorem \ref{thm:main1}(b). We also compare the difference of the mean prediction matrix $\bar{\BZ} = [\BW^{\top}\bar{\bh}_k]_{1\leq k\leq K}$ and the solution $\bar{\BZ}^*$ to the unconstrained feature model with $\lambda_Z =  \sqrt{\lambda_W\lambda_H}$. The datasets and networks are exactly the same as those in Section~\ref{ss:num_nc1}. 

In Figure~\ref{FIG:Block Structure}, we plot the final mean prediction matrix $\bar{\BZ}$ with the $k$-th column being $\bar{\bz}_k = \BW^{\top}\bar{\bh}_k$ over the $12$ experiments computed in Figure~\ref{fig:collapse}. The entries in $\bar{\BZ}$ of the largest magnitude show up on the diagonal. To show a stronger contrast in the plot, we apply min-max standardization across all the mean prediction matrices.
The white dashed lines separate giant clusters into $3\times 3$ blocks which match the setting of Dataset$_1$ and Dataset$_2.$  We see all the entries in each off-diagonal block, and all the off-diagonal entries in each diagonal block share a very similar magnitude in their own block. This indicates the block structure of the mean prediction matrix $\bar{\BZ}$ and also that of the mean feature vectors. 

\begin{figure}[!htb]
    \centering
       \includegraphics[width=175mm]{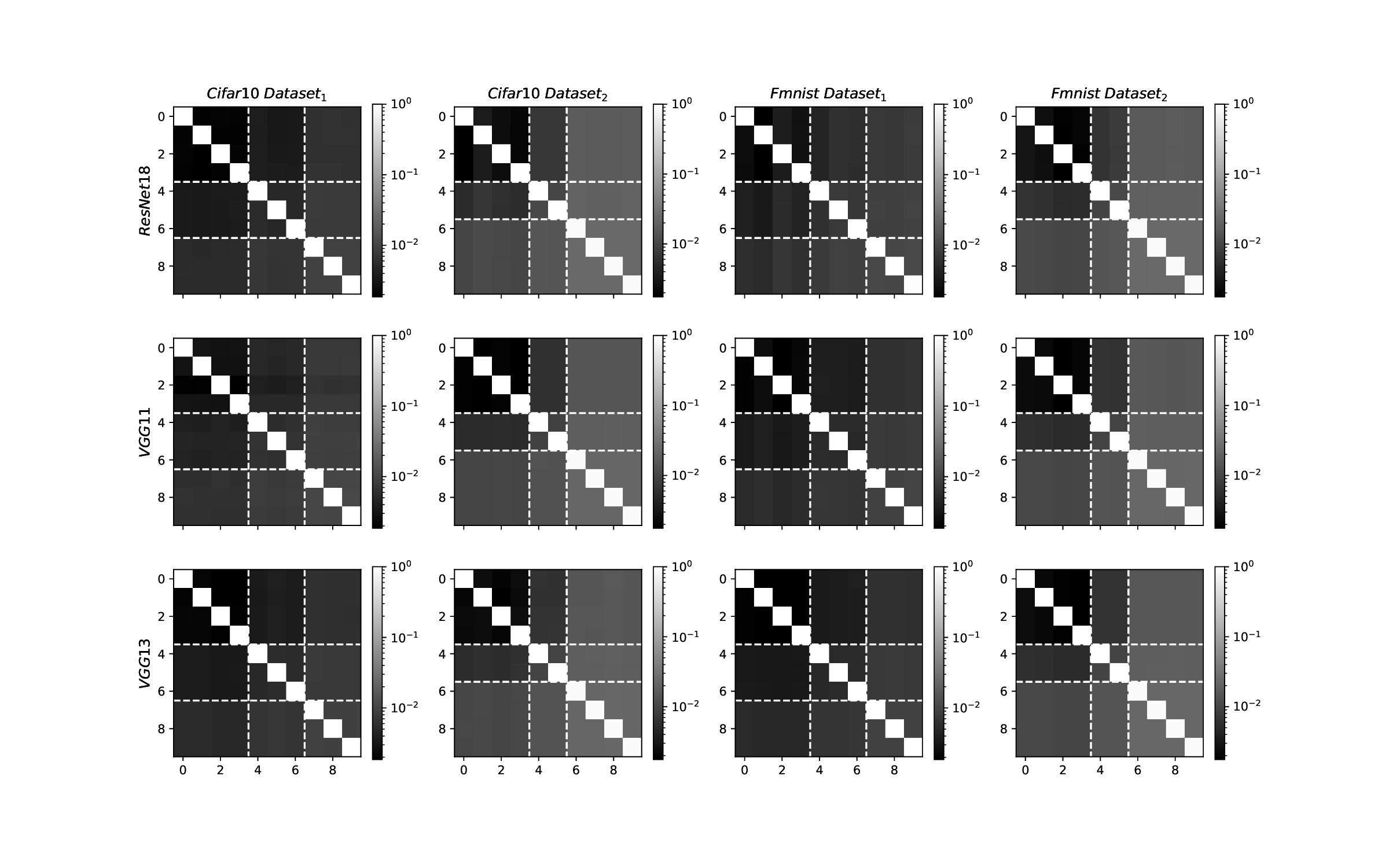}
    \caption{Standardized (over $12$ matrices) mean prediction matrix $\bar{\BZ}$ for all 12 experiments in Figure~\ref{fig:collapse}. The white dashed lines separate clusters $A,B$ and $C$ in Dataset$_1$ and Dataset$_2$. }
    \label{FIG:Block Structure}
\end{figure}

Figure~\ref{FIG:process_diff} and~\ref{FIG:final_diff}, we select two experiments to show the difference between the trained predictions $\bar{\BZ}$ and $\bb$ and the solution $\bar{\BZ}^*$ and $\bb^*$ to~\eqref{eq:ufm}. 
For VGG13 trained on Dataset$_1$, Figure~\ref{FIG:process_diff} (Left) shows a decreasing trend of the relative error between $\bar{\BZ}$ and $\bar{\BZ}^*$ which stabilizes around $0.04$ under both Frobenius and supreme norm. The bias difference is relatively higher and stabilizes around $0.2$; for VGG11 trained on Dataset$_2$, the relative error is higher compared to the previous one, possibly because the Dataset$_2$ is more imbalanced. 
Despite the relative error is approximately 0.1, the final mean prediction matrix shown in Figure~\ref{FIG:final_diff} implies that $\bar{\BZ}$ and $\bar{\BZ}^*$ share a similar block structure.

\begin{figure}[!htb]
    \centering
       \includegraphics[width=0.9\linewidth]{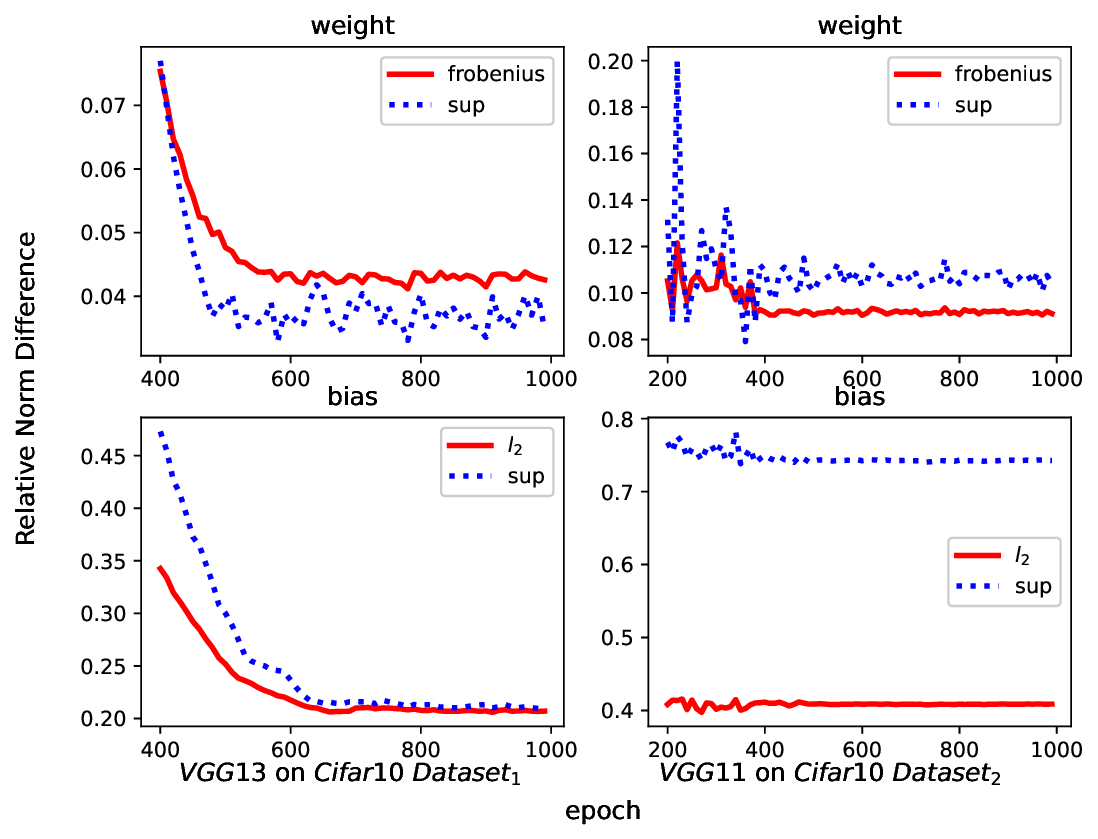}
    \caption{Relative error $\|\bar{\BZ}-\bar{\BZ}^*\|_{F}/\|\bar{\BZ}^*\|_{F}$ ($\|\bb-\bb^*\|_2/\|\bb^*\|_2$) and $\|\bar{\BZ}-\bar{\BZ}^*\|_{\infty}/\|\bar{\BZ}^*\|_{\infty}$ ($\|\bb-\bb^*\|_{\infty}/\|\bb^*\|_{\infty}$) v.s. the epoch for VGG13 on Cifar10 with Dataset$_1$ (Left) and VGG11 on Cifar10 with Dataset$_2$ (Right). The starting epoch numbers are chosen to be $400$ and $200$ when $\mathcal{NC}_1$ has reaches a low level.}
    \label{FIG:process_diff}
\end{figure}

\begin{figure}[!htb]
    \centering
       \includegraphics[width=110mm]{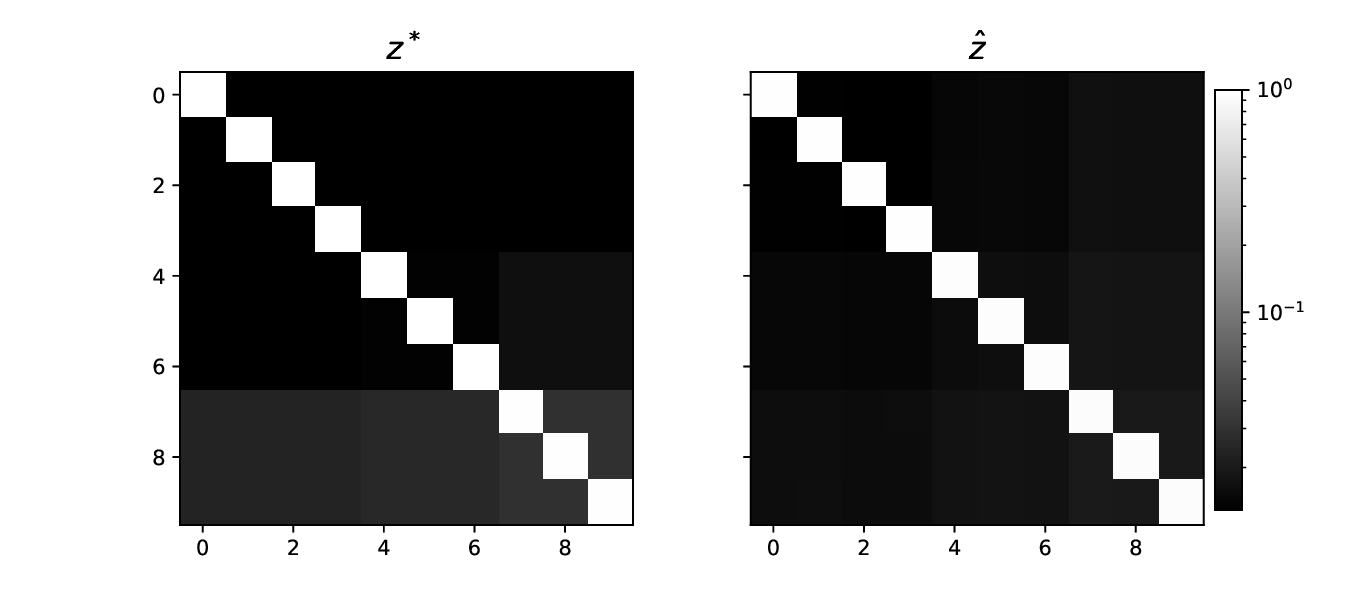} \\
       \includegraphics[width=110mm]{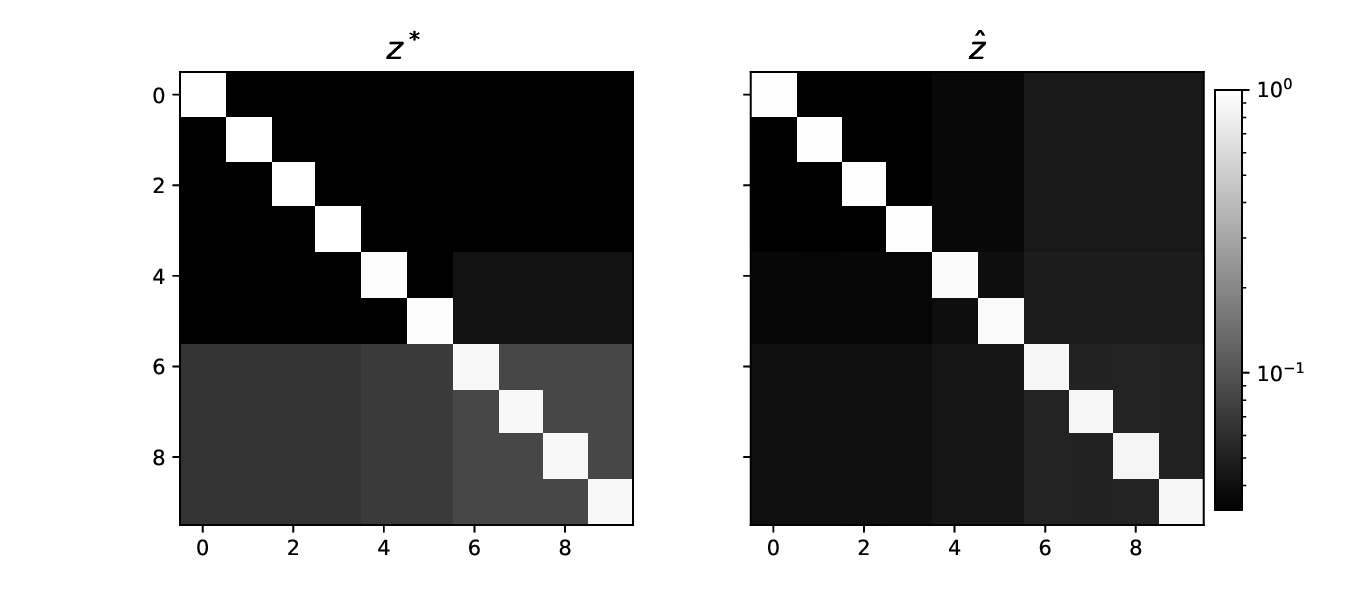}
    \caption{Comparison of the min-max standardized final mean prediction matrix $\bar{\BZ}$ and $\bar{\BZ}^*$. Top: VGG13 on Cifar10 with Dataset$_1$; Bottom: VGG11 on Cifar10, with Dataset$_2$.}
    \label{FIG:final_diff}

\end{figure}

\subsection{Minority collapse}\label{ss:num_nc3}
Theorem~\ref{thm:biasf} and Corollary \ref{corollary:mc} show the sharp threshold on $\lambda_Z$ so that the prediction made by the neural network on the minority classes collapses to a single vector. 
For the experiments below, we adopt a slightly different learning rate scheme. To prevent the features and weights from vanishing due to the large regularization terms, we initialize $\lambda_H$ and $\lambda_W$ to be ten times smaller for the first $1/6$ epochs with stepsize $0.1$. Then we set back the regularization parameters and keep training for the next $1/6$ epochs with the stepsize $0.1$. For the next 1/3 epochs, we set the learning rate as $0.01$. After that, we keep the learning rate equal to $10^{-3}$ for the rest.

We design two types of experiments to verify our theoretical findings. 
The first type fixes $n_A=500$ and $n_B=100$ for a given pair of $(k_A,k_B)$, and $\lambda_b = 0.01$. Then we vary $\lambda_Z$ and run ResNet18 on Cifar10 dataset  for each $\lambda_Z$. For $k_A = k_B = 5$, the results are shown in Figure~\ref{fig:mc1}: it implies that the minority collapse occurs at $\lambda_Z = 0.0033$, which matches $\sqrt{n_B}/N = 1/300$ where $N = k_An_A + k_Bn_B = 3000.$ However, our Theorem~\ref{thm:biasf} fails to predict the threshold beyond which all the predictions become constant: the theoretical threshold is $\sqrt{500}/3000 \approx0.075$ while Figure~\ref{fig:mc1} shows the complete collapse for some $\lambda_Z\leq 0.069$ which is strictly smaller than 0.075.
For $k_A = 3$ and $k_B = 7$, Theorem~\ref{thm:biasf} predicts the minority and complete collapse occur at approximately $\lambda_Z=0.0045$ and $0.0102$ respectively. Figure~\ref{fig:mc3} implies the empirical threshold for minority collapse matches our theoretical prediction while that for the complete collapse is between 0.0086 and 0.0094, strictly smaller than 0.0102.

\begin{figure}[!htb]
    \centering
       \includegraphics[width=160mm]{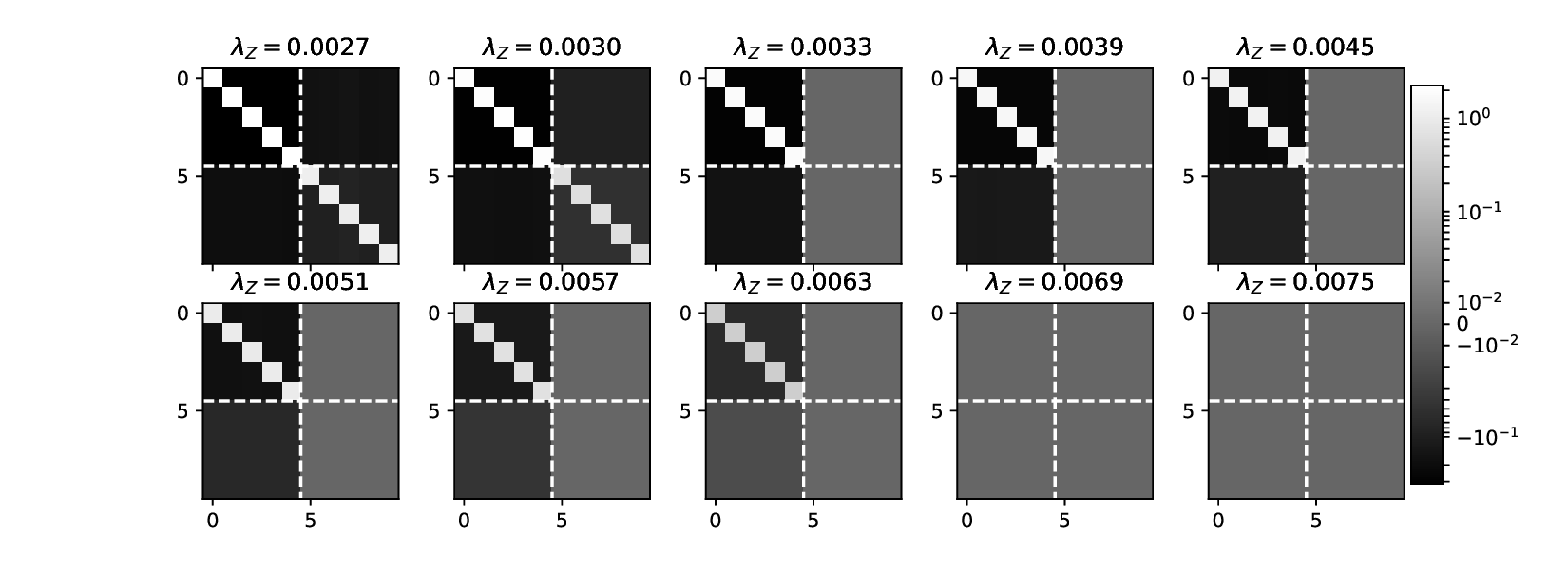}
    \caption{Plot for the mean prediction matrix $\bar{\BZ}$ for $10$ classes with $k_A=k_B=5$ v.s. varying $\lambda_Z$. The white dashed lines separate clusters majority group $A$ and minority group $B$. }
    \label{fig:mc1}

\end{figure}

\begin{figure}[!htb]
    \centering
       \includegraphics[width=160mm]{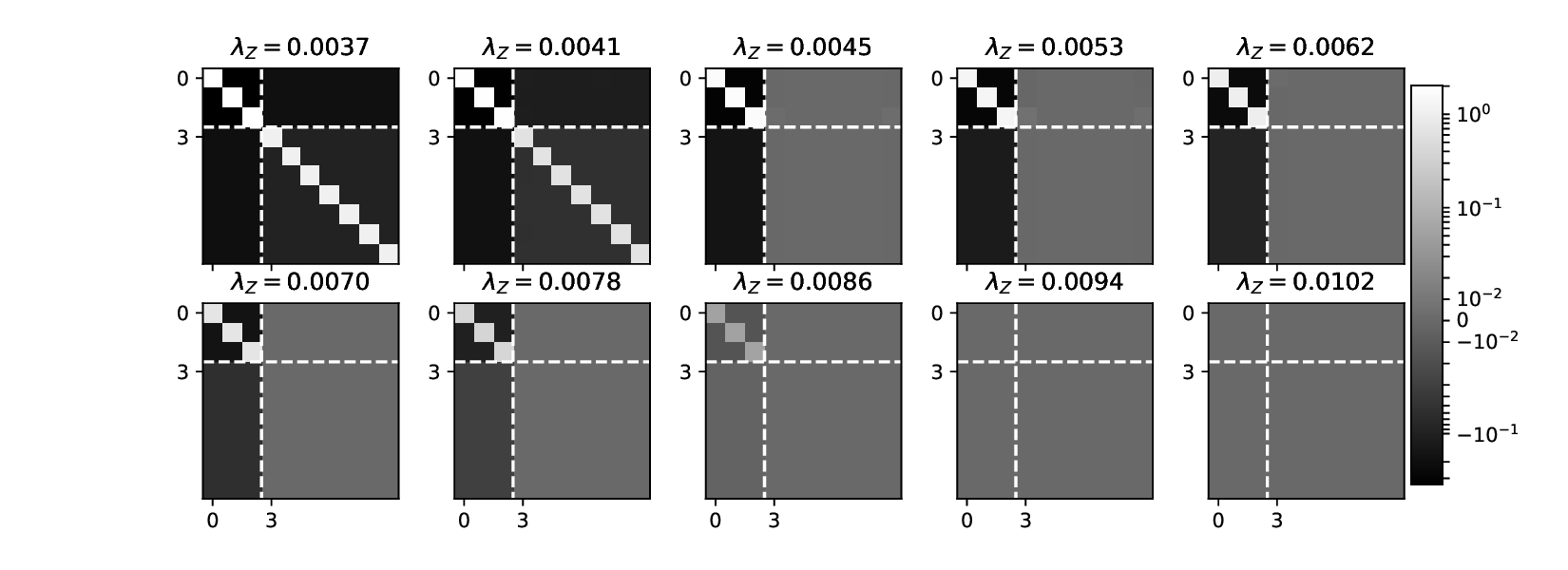}
    \caption{Plot for the mean prediction matrix $\bar{\BZ}$ for $10$ classes with $k_A=3$ and $k_B=7$ v.s. varying $\lambda_Z$. The white dashed lines separate clusters majority group $A$ and minority group $B$. }
    \label{fig:mc3}

\end{figure}
    
In the second type, we fix $\lambda_Z = 0.005,$ $\lambda_b = 0.01$, and $n_B = 100.$ Then we let $n_A$ increase from 100 to 1400, and compute the mean prediction matrix for each set of parameters. For $k_A =k_B=5$, our theory predicts the threshold of $n_A$ for minority and complete collapse are $n_A \approx 300$ and $1392$ respectively. Our numerical experiments in Figure~\ref{fig:mc2} confirm the threshold for minority collapse but the theory overestimates the threshold for complete collapse.
Similar phenomena are also observed for $k_A = 3$ and $k_B=7$ in which the thresholds for $n_A$ are 433 and 3964 respectively, as shown in Figure~\ref{fig:mc4}. 

\begin{figure}[!htb]
    \centering
       \includegraphics[width=160mm]{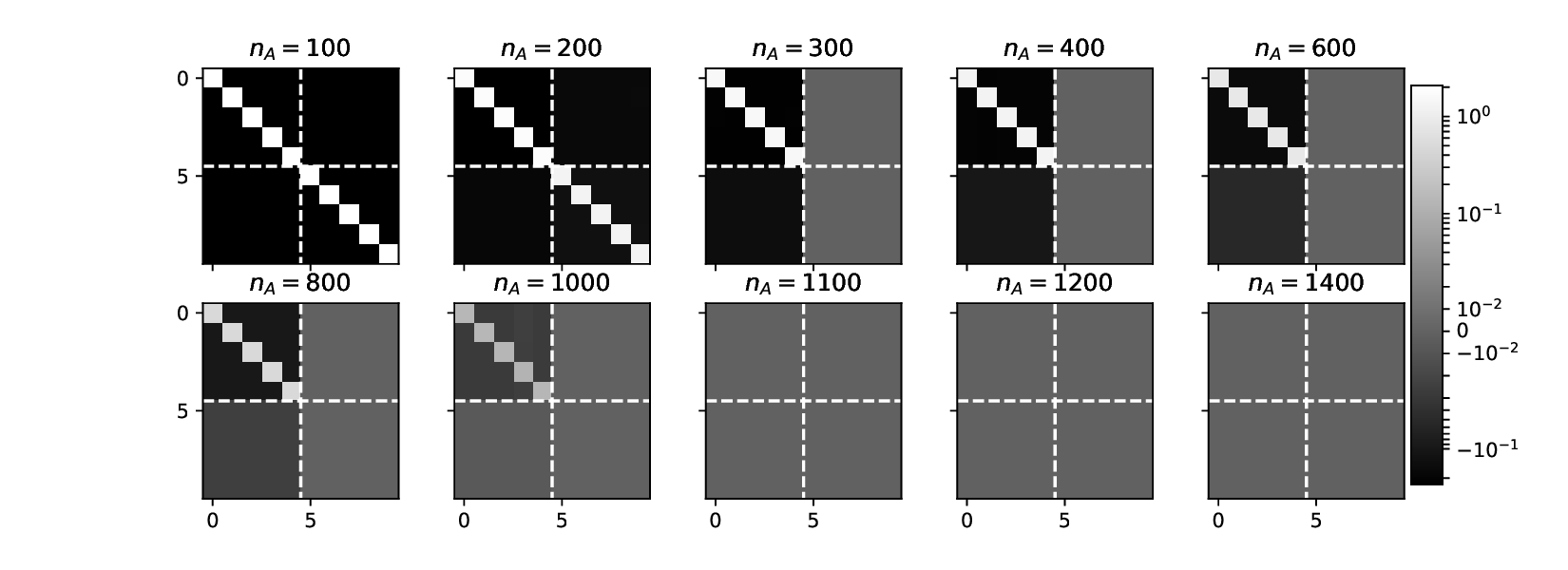}
    \caption{Plot for the mean prediction matrix $\bar{\BZ}$ for $10$ classes with $k_A=k_B=5$ v.s. $n_A$. }
    \label{fig:mc2}

\end{figure}

\begin{figure}[!htb]
    \centering
       \includegraphics[width=160mm]{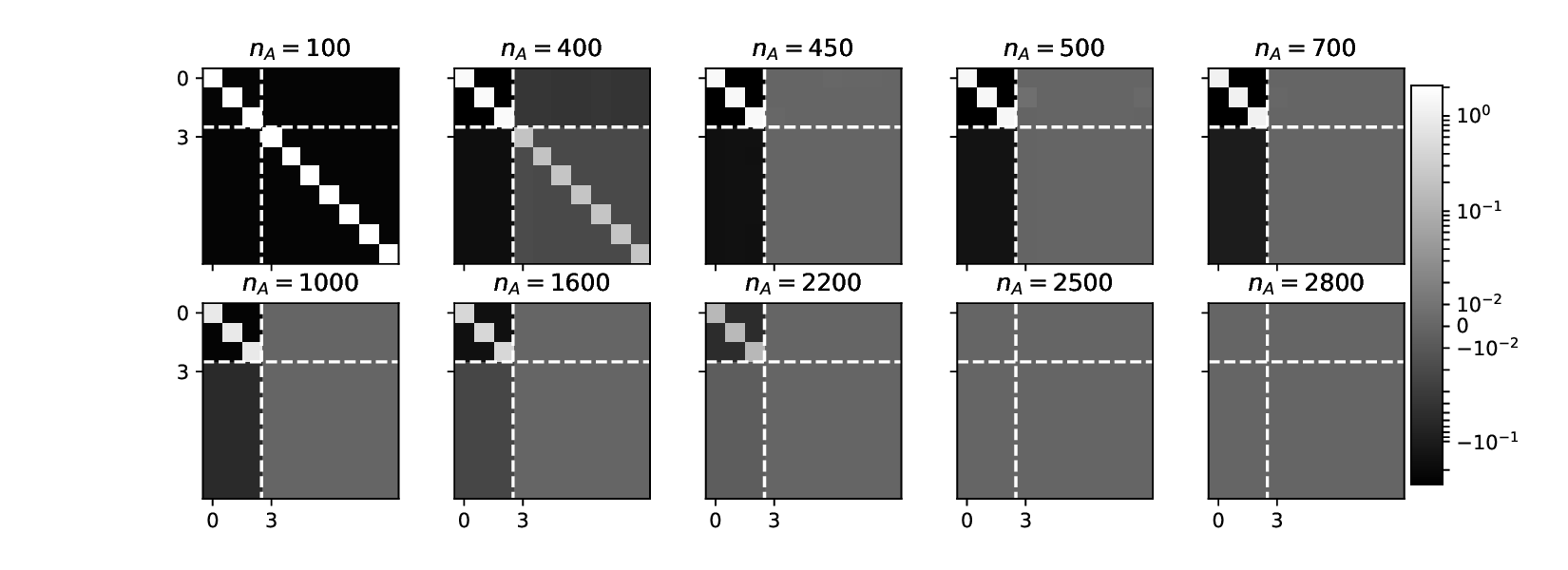}
    \caption{Plot for the mean prediction matrix $\bar{\BZ}$ for $10$ classes with $k_A=3$ and $k_B=7$ v.s. $n_A$. }
    \label{fig:mc4}

\end{figure}

Based on the Figure~\ref{fig:mc1}-\ref{fig:mc4}, we make the following main observations:  (i) the threshold of minority collapse $\lambda_Z = \sqrt{n_B}/N$ matches the empirical experiments. However, the theoretical threshold for complete collapse $\lambda_Z = \sqrt{n_A}/N$ tends to overestimate; (ii) all four figures confirm the block structure of the mean prediction matrix that is characterized by cases (a), (b), and (d) in Theorem~\ref{thm:biasf}. For case (c), we can take a look at the 8th subfigure ($\lambda_Z=0.0063$) in Figure~\ref{fig:mc1}. The right blocks and lower left block have a much smaller magnitude compared with the upper left blocks. However, the entries in the lower left blocks (the order is $10^{-3}$) still are much larger than those on the right block (the order is $10^{-7}$). 
Therefore, we do not see a strong signal of the case (c) for $\lambda_Z$ before the complete collapse occurs.


\subsection{Convergence to the ETF}\label{ss:num_nc4}
In this section, we will carry out some experiments for Theorem~\ref{thm:limit}. For the parameters, we fix parameters $k_A = 5,k_B = 5, r=n_A/n_B=2, N\lambda_Z=0.1$, and $\lambda_b = 0.01$. We train ResNet18 on the Cifar10 dataset with different $n_A$: $n_A$ ranges from $500$ to $5500$. For each set of parameters, we run 2000 epochs and compute the pairwise correlation (i.e., cosine angle) for mean prediction vectors $\bar{\BZ}$, i.e., 
\[
\widehat{\bf\Theta} := \diag(\bar{\BZ}^{\top}\bar{\BZ})^{-1/2}\bar{\BZ}^{\top}\bar{\BZ} \diag(\bar{\BZ}^{\top}\bar{\BZ})^{-1/2}.
\]

Due to the block structure of $\bar{\BZ}$, the variance of angles in each block of $\widehat{\bf \Theta}$ is quite small, and thus here, we only plot the mean within-class correlation for $A$ and $B$ respectively, and the mean correlation between $A$ and $B$.
Here $K = k_A +k_B = 10$, and the pairwise correlation is $-1/9$ for the ETF. Figure~\ref{fig:limit} shows a clear convergence of all the three groups of mean correlation toward $-1/9$ as $n_A$ increases, i.e., the mean prediction vectors slowly converge to an ETF. This validates our result in Theorem~\ref{thm:limit}, i.e., the impact created by the imbalance in data size on the prediction of neural networks diminishes as the number of training samples increases.
\begin{figure}[!htb]
    \centering
       \includegraphics[width=0.7\linewidth]{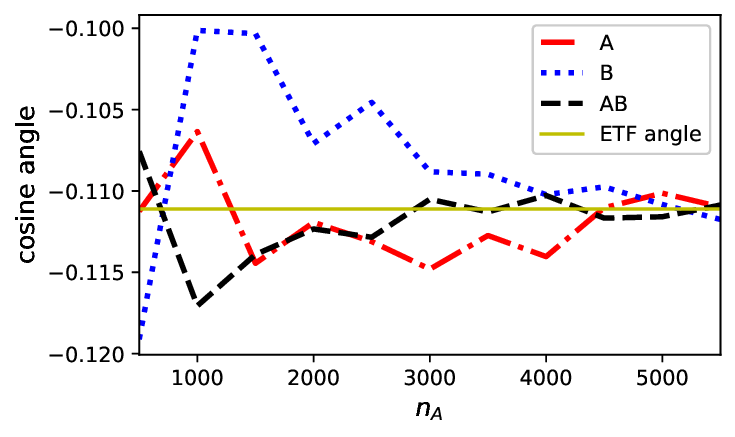}
    \caption{$y$-axis plots the mean pairwise correlation for mean prediction vector $\bar{\BZ}$ within the cluster $A$ (red dot-dashed) and $B$ (blue dotted), and between clusters $A$ and $B$ (black dashed); $x$-axis is the size of $n_A$ between 500 and 5500. The pairwise correlation of the ETF for $10$ classes is denoted by the yellow straight line.}
    \label{fig:limit}
\end{figure}


\section{Proof}\label{s:proof}

\subsection{Basic facts and optimality condition}
This subsection establishes important lemmas that will be used for proving our main theorems.
\begin{lemma}\label{lem:varphi}
Define
\begin{equation}\label{def:varphi}
\varphi(\BZ,\bb) = \frac{1}{N}\ell_{CE}(\BZ+\bb\bone_N^{\top},\BY) = \frac{1}{N}\sum_{k=1}^K\sum_{i=1}^{n_k} \ell_{CE}(\bz_{ki}+\bb,\be_k)
\end{equation}
and then $\varphi(\BZ,\bb)$ is strongly convex in the direction $(\BDelta_Z,\BDelta_b)\in \RR^{K\times N}\oplus\RR^K$ that belongs to $\{(\BDelta_Z,\BDelta_b): \bone_K^{\top}(\BDelta_Z + \BDelta_b\bone_N^{\top}) =0 \}$.
\end{lemma}
\begin{proof}
The proof is straightforward, and it suffices to show the quadratic form 
\[
[\BDelta_Z, \BDelta_b] : \nabla_{\BZ,\bb}^2 \varphi(\BZ,\bb) : [\BDelta_Z,\BDelta_b] \geq \lambda(\BZ,\bb) \cdot \left\| \BDelta_Z + \BDelta_b\bone_N^{\top}\right\|_F^2,
\]
for every $(\BDelta_Z,\BDelta_b)$ satisfying $\bone_K^{\top}(\BDelta_Z + \BDelta_b\bone_N^{\top}) = 0$
where $\lambda(\BZ,\bb)$ is a strictly positive number that only depends on $(\BZ,\bb)$. For ease of notation, define
\begin{equation}\label{def:pki}
\bp_{ki} := \frac{\exp(\bz_{ki} +\bb)}{\lag \exp(\bz_{ki} +\bb), \bone_K\rag},~~~\BP = [\bp_{ki}]_{1\leq i\leq n_k,1\leq k\leq K}\in\RR^{K\times N}
\end{equation}
as the probability vector associated with $\bz_{ki}+\bb.$
The gradient of $\varphi$ is
\[
\frac{\pa \varphi}{\pa\bz_{ki}} = \frac{1}{N}(\bp_{ki} - \be_k),~~~\frac{\pa \varphi}{\pa \bb} = \frac{1}{N} \sum_{k,i} (\bp_{ki} - \be_k)
\] 
whose matrix form is
\[
\frac{\pa \varphi}{\pa \BZ} = \frac{1}{N} (\BP - \BY),~~~\frac{\pa \varphi}{\bb} = \frac{1}{N} (\BP - \BY)\bone_N.
\]
The corresponding Hessian is 
\begin{align*}
& \frac{\pa^2 \varphi}{\pa \bz_{ki}^2} = \frac{1}{N}\left( \diag(\bp_{ki}) - \bp_{ki}\bp_{ki}^{\top}\right),~~~\frac{\pa^2 \varphi}{\pa \bz_{ki}\pa\bz_{k'i'}} = 0,~~\forall (k,i)\neq (k',i'), \\
& \frac{\pa^2\varphi}{\pa \bb^2} = \frac{1}{N} \sum_{k,i} \left(\diag(\bp_{ki}) - \bp_{ki}\bp_{ki}^{\top}\right),~~~
\frac{\pa^2\varphi}{\pa\bz_{ki}\pa \bb} = \frac{1}{N} (\diag(\bp_{ki}) - \bp_{ki}\bp_{ki}^{\top}).
\end{align*} 
Note that $\bp_{ki}\bp_{ki}^{\top}$ is a positive matrix and thus the associated Laplacian $\diag(\bp_{ki}) - \bp_{ki}\bp_{ki}^{\top}$ is positive semidefinite with its second smallest eigenvalue strictly positive. Let $\BDelta_{Z,ki}$ be the difference in the variable $\bz_{ki}$, and then the quadratic form equals
\begin{align*}
& [\BDelta_Z, \BDelta_b] : \nabla_{\BZ,\bb}^2 \varphi(\BZ,\bb) : [\BDelta_Z, \BDelta_b] \\
& =  \frac{1}{N} \sum_{k,i} (\BDelta_{Z,ki} + \BDelta_b)^{\top} \left(\diag(\bp_{ki}) - \bp_{ki}\bp_{ki}^{\top}\right) (\BDelta_{Z,ki} + \BDelta_b) \\
& \geq  \frac{1}{N} \sum_{k,i} \lambda_2(\diag(\bp_{ki}) - \bp_{ki}\bp_{ki}^{\top}) \| \BDelta_{Z,ki} + \BDelta_b\|^2 \\
& \geq  \frac{1}{N} \min_{k,i} \lambda_2(\diag(\bp_{ki}) - \bp_{ki}\bp_{ki}^{\top}) \cdot \| \BDelta_Z + \BDelta_b \bone_N^{\top}\|_F^2 
\end{align*}
where the first inequality follows from the fact that $\left(\diag(\bp_{ki}) - \bp_{ki}\bp_{ki}^{\top}\right)\bone_K = 0$ and $\lag \bone_K, \BDelta_{Z,ki} +\bb\rag = 0$ for all $1\leq i\leq n_k$ and $1\leq k\leq K$.

\end{proof}

\begin{lemma}[\bf Optimality condition]\label{lem:opt} 
The first-order optimality condition of ${\cal L}(\BZ,\bb)$ in~\eqref{eq:cvx} is
\begin{equation}\label{cond:1st}
N^{-1}(\BY - \BP)\in\lambda_Z\pa \|\BZ\|_*,~~~N^{-1}(\BY - \BP)\bone_N = \lambda_b \bb
\end{equation}
where $\BP = [\bp_{ki}]_{1\leq i\leq  n_k,1\leq k\leq K}$, $\bp_{ki}$ are defined in~\eqref{def:pki}, and $\pa \|\BZ\|_*$ stands for the subdifferential of the nuclear norm at $\BZ.$
In particular, the global minimizer $(\BZ,\bb)$ satisfies $\bone_K^{\top}\BZ = 0$ and $\bone_K^{\top}\bb = 0.$
\end{lemma}
\begin{proof}
Consider ${\cal L}(\BZ,\bb) = N^{-1}\ell_{CE}(\BZ+\bb\bone_N^{\top},\BY) + \lambda_Z\|\BZ\|_* + \lambda_b\|\bb\|^2/2$ in~\eqref{eq:cvx}.
Then its gradient (subgradient) is
\[
\frac{\pa L}{\pa \BZ} = N^{-1} \left(\BP - \BY\right) + \lambda \pa \|\BZ\|_*,~~\frac{\pa \varphi}{\pa \bb} = N^{-1}\left( \BP - \BY\right)\bone_N +\lambda_b \bb.
\]
Therefore, $(\BZ,\bb)$ is a global minimizer if
\[
N^{-1}(\BY - \BP)\in \lambda_Z \pa\|\BZ\|_*,~~~N^{-1}(\BY - \BP)\bone_N = \lambda_b \bb 
\]
where $\pa \|\BZ\|_*$ is the subdifferential of nuclear norm at $\BZ.$

For any $\BZ$ and $\bb$, we notice that $\BZ + \bone_K\bv^{\top}$ and $\bb+\mu\bone_K$ does not change $\varphi(\BZ,\bb)$ for any $\bv\in\RR^N$ and $\mu\in\RR$, i.e., 
\[
\varphi(\BZ+ \bone_K\bv^{\top},\bb+\mu\bone_K) = \varphi(\BZ,\bb).
\] 
Note that
\[
\|\bb\|^2 \geq \min_{\mu} \|\bb - \mu\bone_K\|^2 = \| \bb - \bone_K \lag \bone_K,\bb\rag/K\|^2
\]
and
\[
\min_{\bv\in\RR^N} \|\BZ - \bone_K\bv^{\top} \|_* = \left\| (\I_K - \BJ_K/K)\BZ\right\|_* 
\]
where the subdifferential of $\|\BZ - \bone_K\bv^{\top} \|_*$ is $(\pa \|\BZ - \bone_K\bv^{\top}\|_*)^{\top}\bone_K$ and 
\[
0 \in (\pa \|\BZ - \bone_K\bv^{\top}\|_*)^{\top}\bone_K \Big|_{\bv = \BZ^{\top}\bone_K/K}
\]
since the column space of $(\I_K - \BJ_K/K)\BZ$ is perpendicular to $\bone_K.$ Therefore, the global minimizer must satisfy $\bone_K^{\top} \BZ =0$ and $\bone_K^{\top}\bb = 0$.
\end{proof}

By considering the exact form of $\|\BZ\|_{*}$, the optimality condition in Lemma~\ref{lem:opt} can be expressed explicitly as given by the next corollary.

\begin{corollary}\label{Cor:Op}
Assume $\bz_{ki} = \bar{\bz}_k$ for $1\leq i\leq n_k,1\leq k\leq K$ and $\lag \bar{\bz}_k, \bone_K\rag = 0$, and then it holds that
\begin{equation}\label{cond:1stred}
\begin{aligned}
N^{-1}(\I_K - \bar{\BP}) & =\lambda_Z \left( \left[\left(\bar{\BZ}\BD\bar{\BZ}^{\top} \right)^{\dagger}\right]^{1/2}\bar{\BZ} + \bar{\BR} \right), \\
 N^{-1}(\I_K - \bar{\BP})\bn &= \lambda_b \bb
\end{aligned}
\end{equation}
where $\bar{\BZ} = [\bar{\bz}_1,\cdots,\bar{\bz}_K]\in\RR^{d\times K}$, $\bar{\BR}$ satisfies $\bar{\BR}\BD\bar{\BZ}^{\top} = 0$, $ \bar{\BZ}^{\top}\bar{\BR} = 0$, $\|\bar{\BR}\BD^{1/2}\| \leq 1$, and $\BD$ is defined in~\eqref{def:D}. In particular, if $\bar{\BZ}$ is of rank $K-1$, then
\[
N^{-1}(\I_K - \bar{\BP}) =\lambda_Z \left[\left(\bar{\BZ}\BD\bar{\BZ}^{\top} \right)^{\dagger}\right]^{1/2}\bar{\BZ}.
\]
\end{corollary}

\begin{proof}
Under assumption, $\bz_{ki} = \bar{\bz}_k,~1\leq i\leq n_k$ and $\bar{\bz}_k^{\top}\bone_K = 0$, we have $\BZ = \bar{\BZ}\BY$ and $\BP = \bar{\BP} \BY$ where $\bar{\BZ}$ and $\BY$ are defined in~\eqref{def:barZS} and $\bar{\BP} = [\bar{\bp}_1,\cdots,\bar{\bp}_K]$ with $\bar{\bp}_k$ as the probability vector w.r.t. $\bar{\bz}_k +\bb.$
Then~\eqref{cond:1st} reduces to
\begin{equation}\label{eq:1stv1}
N^{-1} (\I_K - \bar{\BP})\BY \in \lambda_Z \pa \| \bar{\BZ}\BY \|_*
\end{equation}
Here we let $\BZ =\BU\BSigma\BV^{\top}$ be the SVD of $\BZ$.
\[
\pa \| \bar{\BZ}\BY  \|_* = \left\{ \BU\BV^{\top} + \BR: \|\BR\|\leq 1,~\BU^{\top}\BR = 0,~\BR\BV = 0 \right\}
\]
where 
\[
\BU\BV^{\top} = \left[\left(\bar{\BZ}\BD\bar{\BZ}^{\top} \right)^{\dagger}\right]^{1/2} \bar{\BZ}\BY
\]
and $^{\dagger}$ denotes the Moore-Penrose pseudo-inverse. 
Then~\eqref{eq:1stv1} becomes
\[
N^{-1} (\I_K - \bar{\BP})\BY = \lambda_Z \left( \left[\left(\bar{\BZ}\BD\bar{\BZ}^{\top} \right)^{\dagger}\right]^{1/2}\bar{\BZ}\BY + \BR  \right)
\]
where $\BR$ is in the form of
\[
\BR = \bar{\BR}\BY,~~\bar{\BR} = [\bar{\br}_1,\cdots,\bar{\br}_K]
\]
such that
\begin{align*}
& \bar{\BR}\BY\left(\bar{\BZ}\BY\right)^{\top} = \bar{\BR}\BD\bar{\BZ}^{\top} = 0, \quad \bar{\BZ}^{\top}\bar{\BR} = 0,~~~\|\bar{\BR}\BD^{1/2}\| \leq 1
\end{align*}
where $\BD$ is defined in~\eqref{def:D} and $\BY\BY^{\top} = \BD.$
This leads to
\[
N^{-1} (\I_K - \bar{\BP}) = \lambda_Z \left( \left[\left(\bar{\BZ}\BD\bar{\BZ}^{\top} \right)^{\dagger}\right]^{1/2}\bar{\BZ} + \bar{\BR} \right)
\]
where $\bar{\BR}\BD\bar{\BZ}^{\top} = 0$ and $ \bar{\BZ}^{\top}\bar{\BR} = 0$.

In particular, if $\bar{\BZ}$ is of rank $K-1,$ then each column of $\bar{\BR}$ is parallel to $\bone_K$ since $\bone_K^{\top}\bar{\BZ} = 0$. Since $\bar{\BP}$ is a positive left-stochastic matrix with $\bar{\BP}^{\top} \bone_K= \bone_K$, and thus $\I_K-\bar{\BP}$ is of rank $K-1$.
Note that $\bone_K$ is also in the left null space of $\I_K- \bar{\BP}$, then $\bar{\BR} = 0.$ It means
\[
N^{-1}(\I_K - \bar{\BP}) =\lambda_Z \left[\left(\bar{\BZ}\BD\bar{\BZ}^{\top} \right)^{\dagger}\right]^{1/2} \bar{\BZ}.
\]
For $\bb$, it is straightforward to have
\[
N^{-1}(\BY - \BP)\bone_N = N^{-1}\sum_{k=1}^K n_k (\be_k - \bar{\bp}_k) = N^{-1}(\I_K - \bar{\BP})\bn= \lambda_b \bb.
\]
\end{proof}


\begin{lemma}[\bf Invariance under permutation]\label{lem:perm}
\quad
\begin{enumerate}[(a)]
\item Suppose $\BX$ satisfies
\[
\BPi \BX\BPi' = \BX
\]
for any permutation $\BPi$ and $\BPi'$, then $\BX= c\BJ$ for some $c$. \item Suppose $\BX$ satisfies
\[
\BPi \BX\BPi^{\top} = \BX, 
\]
then $\BX = a\I + c\BJ$ for some $a$ and $c.$  
\end{enumerate}
\end{lemma}
\begin{proof}
For (a), we have
\[
\BPi\BX = \BX,~~~\BX\BPi' = \BX
\]
holds for any permutation matrix. Then $\BX$ is invariant under either row or column permutation. Therefore, $\BX$ is a constant matrix.

For (b), we first pick the permutation matrix $\BPi_{\ell,\ell'}$ that only exchanges $\ell$- and $\ell'$-th entries. Then $\BPi_{\ell,\ell'} = \BPi_{\ell',\ell}$ and
\[
\BPi_{\ell,\ell'} \BX \BPi_{\ell,\ell'} = \BX.
\]
Now let $\be_k$ be any one-hot vector, and it holds
\[
X_{\ell'\ell'} = \be_{\ell'}^{\top}\BX\be_{\ell'}=\be_{\ell}^{\top}\BPi_{\ell,\ell'} \BX \BPi_{\ell,\ell'}\be_{\ell} = \be_{\ell}^{\top}\BX \be_{\ell}  = X_{\ell\ell},
\] 
where $\BPi_{\ell,\ell'}\be_{\ell'} = \be_{\ell}$. This implies that the diagonal entries of $\BX$ are the same.

For any $k\neq \ell$ or $\ell'$, $\BPi_{\ell,\ell'}\be_k = \be_k$ holds and
\[
X_{k\ell} = \be_k^{\top}\BX\be_{\ell} = \be_k^{\top}\BPi_{\ell,\ell'} \BX \BPi_{\ell,\ell'} \be_{\ell} = \be_k^{\top}\BX \be_{\ell'} = X_{k\ell'}.
\]
Similarly, it also holds $X_{\ell k} = X_{\ell'k}$.
Therefore, $X_{\ell\ell'}$ is the same for $\ell\neq \ell'$.
\end{proof}

\subsection{Proof of Theorem~\ref{thm:main1}}

\paragraph{Theorem~\ref{thm:main1}(a)} The proof follows from a convenient~\emph{permutation argument}.
Recall~\eqref{eq:cvx} equals
\begin{equation}\label{OBJ_ORI}
\begin{aligned}
{\cal L}(\BZ,\bb) & = \frac{1}{N}\sum_{k=1}^K\sum_{i=1}^{n_k} \ell_{CE}(\bz_{ki}+\bb,\be_k) + \lambda_Z \|\BZ\|_* + \frac{\lambda_b}{2} \|\bb\|^2
\end{aligned}
\end{equation}
Note that it is strongly convex in $(\BZ,\bb)$ in the restricted direction, as shown in Lemma~\ref{lem:varphi}. Then the sublevel set of $\{(\BZ,\bb): {\cal L}(\BZ,\bb) \leq c\}$ is a compact set when restricted on those directions, for any $c$ if $\lambda_Z>0$ and $\lambda_b > 0$. Therefore, the existence of a global minimizer is guaranteed. 
Suppose $(\BZ,\bb)$ is a global minimizer to ${\cal L}(\BZ,\bb)$ and we know from Lemma~\ref{lem:opt} that $\bone_K^{\top}\BZ= 0$ and $\bone_K^{\top}\bb = 0$.
Then by using Jensen's inequality, we have
\begin{align*}
\frac{1}{n_k} \sum_{i=1}^{n_k} \ell_{CE}(\bz_{ki}+\bb,\be_k) \geq \ell_{CE}\left( \frac{1}{n_k} \sum_{i=1}^{n_k}(\bz_{ki}+\bb),\be_k   \right) = \ell_{CE}(\bar{\bz}_k+\bb,\be_k) 
\end{align*}
where $\{\bz_{ki}\}_{i=1}^{n_k}$ belong to the same class and $\bar{\bz}_k = n_k^{-1}\sum_{i=1}^{n_k}\bz_{ki}.$

Let $\blkdiag(\BPi_{n_1},\cdots,\BPi_{n_k})$ be a block-diagonal permutation matrix where each $\BPi_{n_k}$ is any $n_k\times n_k$ permutation. Then
\[
\| \BZ\blkdiag(\BPi_{n_1},\cdots,\BPi_{n_k})\|_* = \|\BZ\|_*.
\]
Note that
\[
\frac{1}{n_1!\cdots n_K!}\sum_{\{\BPi_k,1\leq k\leq K\}}\blkdiag(\BPi_{n_1},\cdots,\BPi_{n_k}) = \blkdiag(\BJ_{n_1}/n_1,\cdots,\BJ_{n_K}/n_K)
\]
where the number of distinct permutation matrices in the form of $\blkdiag(\BPi_{n_1},\cdots,\BPi_{n_k})$ is $n_1!\cdots n_K!$ and the summation is taken over all possible permutations in that form.
Using Jensen's inequality again results in
\begin{equation}
\|\BZ\|_* \geq \|\BZ\blkdiag(\BJ_{n_1}/n_1,\cdots,\BJ_{n_K}/n_K)\|_* = \| [\bar{\bz}_1\bone_{n_1}^{\top},\cdots,\bar{\bz}_K\bone_{n_K}^{\top}] \|_* = \|\bar{\BZ}\BY\|_*
\end{equation}
As a result, it holds that
\begin{align*}
{\cal L}(\BZ,\bb) & = \frac{1}{N}\sum_{k=1}^Kn_k \left(\frac{1}{n_k}\sum_{i=1}^{n_k} \ell_{CE}(\bz_{ki}+\bb,\be_k)\right) + \lambda_Z \|\BZ\|_* + \frac{\lambda_b}{2} \|\bb\|^2 \\
& \geq \sum_{k=1}^K \frac{n_k}{N} \ell_{CE}(\bar{\bz}_k + \bb,\be_k) + \lambda_Z \| [\bar{\bz}_1\bone_{n_1}^{\top},\cdots,\bar{\bz}_K\bone_{n_K}^{\top}] \|_* + \frac{\lambda_b}{2}\|\bb\|^2  \\
& = \frac{1}{N}\ell_{CE}(\bar{\BZ}\BY +\bb\bone_N^{\top}, \BY) + \lambda_Z\|\bar{\BZ}\BY\|_* + \frac{\lambda_b}{2}\|\bb\|^2 = R(\bar{\BZ}\BY,\bb).
\end{align*}
As Lemma~\ref{lem:varphi} guarantees the strong convexity of ${\cal L}(\BZ,\bb)$ restricted on $\{(\BDelta_Z,\BDelta_b): \bone_K^{\top}(\BDelta_Z +\BDelta_b\bone_N^{\top}) = 0\}$, the global optimality of $(\BZ,\bb)$ implies that $\BZ=\bar{\BZ}\BY$, i.e., $\bz_{ki} = \bar{\bz}_k$ for $1\leq i\leq n_k$.

\paragraph{Theorem~\ref{thm:main1}(b)}
Without loss of generality, we let $n_1\leq n_2\leq\cdots\leq n_K$ and $\Gamma_j = \{k: n_k = N_j,~1\leq k\leq K\}$ where $\{N_j\}_{j=1}^J$ are the distinct values of $\{n_k\}_{k=1}^K$. In other words, we re-group each class according to their individual class sizes and it holds $K = \sum_{j=1}^J |\Gamma_j|$. As Theorem~\ref{thm:main1}(a) implies the global minimizer $\BZ$ satisfies $\bz_{ki} = \bar{\bz}_k$,~$1\leq i\leq n_k$. Then it holds
\begin{align*}
{\cal L}(\BZ,\bb) & = \frac{1}{N} \sum_{k=1}^K n_k \ell_{CE}(\bar{\bz}_k+\bb,\be_k) + \lambda_Z \| \bar{\BZ} \BY\|_* + \frac{\lambda_b}{2}\|\bb\|^2 \\
& = \frac{1}{N} \sum_{j=1}^J N_j \sum_{k\in \Gamma_j} \ell_{CE}(\bar{\bz}_k +\bb,\be_k) + \lambda_Z \| \bar{\BZ} \BD^{1/2}\|_* + \frac{\lambda_b}{2}\|\bb\|^2
\end{align*}
where $\BD = \BY\BY^{\top}$ and 
\[
\bar{\BZ} = [\BZ_{\Gamma_1},\cdots,\BZ_{\Gamma_J}],~~\BZ_{\Gamma_j} = [\bar{\bz}_k]_{k\in\Gamma_j}.
\]
Let $\BPi_{\ell,\ell'}$ be a $K\times K$ permutation matrix that only switches the index $\ell$ and $\ell'$ but keeps the other indices unchanged. Then we claim that
\[
{\cal L}(\BPi_{\ell,\ell'}\bar{\BZ}\BPi_{\ell,\ell'}\BY,\BPi_{\ell,\ell'}\bb) = {\cal L}(\bar{\BZ}\BY,\bb)
\]
as long as index $\ell$ and $\ell'$ are in the same $\Gamma_j$ for some $j.$
By using the restricted strong convexity of ${\cal L}(\BZ,\bb)$ in Lemma~\ref{lem:varphi}, we have
\begin{equation}\label{eq:claim_perm}
\BPi_{\ell,\ell'}\bar{\BZ}\BPi_{\ell,\ell'}=\bar{\BZ},~~~\BPi_{\ell,\ell'}\bb = \bb,~~\forall \ell,\ell'\in\Gamma_j,~\text{ for some }j.
\end{equation}

Now we assume the claim holds, and see how it leads to the desired result.
For any $j\neq j'$, it holds that
$\BPi_{\Gamma_{j}}^{\top} \bar{\BZ}_{j,j'} \BPi_{\Gamma_{j'}} = \bar{\BZ}_{j,j'}$
where $\BPi_{\Gamma_{j}}$ is any $k_j\times k_j$ permutation matrix. 
Lemma~\ref{lem:perm} implies every entry in $\bar{\BZ}_{j,j'}$ equals some constant $a_{jj'}$. 
For the diagonal blocks, we have $\BPi_{\Gamma_{j}}^{\top} \bar{\BZ}_{j,j} \BPi_{\Gamma_{j}} = \bar{\BZ}_{j,j}$.
Therefore, $\bar{\BZ}_{j,j} = a_j\I_{\Gamma_j} + a_{jj'}\bone_{\Gamma_j}\bone_{\Gamma_j}^{\top}$ follows from Lemma~\ref{lem:perm}.
Now we conclude that
\begin{equation}\label{block_struct}
\bar{\BZ} = \sum_{j=1}^J a_j \I_{\Gamma_j} + \sum_{1\leq j,j'\leq J} a_{jj'} \bone_{\Gamma_j}\bone_{\Gamma_{j'}}^{\top},~~~ a_j + \sum_{j'=1}^J a_{j'j} |\Gamma_{j'}| = 0
\end{equation}
where $\I_{\Gamma_j} = \diag(\bone_{\Gamma_j})$, which follows from Lemma~\ref{lem:perm}.

From $\BPi_{\ell,\ell'}\bb = \bb$ for any $\ell,\ell'\in\Gamma_j$, then $\bb$ is in the form of
\[
\bb = \sum_{j=1}^J c_j \bone_{\Gamma_j},~~\sum_{j=1}^J c_j |\Gamma_j| = 0.
\]

To complete the proof, it remains to justify the claim. For any $k\neq \ell$ or $\ell'$, then
\begin{align*}
\ell_{CE}([\BPi_{\ell,\ell'}(\bar{\BZ}+\bb\bone_K^{\top})\BPi_{\ell,\ell'}]_k,\be_k) & = \ell_{CE}(\BPi_{\ell,\ell'}(\bar{\BZ}+\bb\bone_K^{\top})\BPi_{\ell,\ell'}\be_k,\be_k) \\
& = \ell_{CE}(\BPi_{\ell,\ell'}(\bar{\bz}_k +\bb),\be_k) = \ell_{CE}(\bar{\bz}_k+\bb,\be_k)
\end{align*}
where $\BPi_{\ell,\ell'}\be_k = \be_k$ if $k\neq \ell$ or $\ell'.$ In addition, it holds
\begin{align*}
\sum_{k\in\{\ell,\ell'\}}\ell_{CE}([\BPi_{\ell,\ell'}(\bar{\BZ}+\bb\bone_K^{\top})\BPi_{\ell,\ell'}]_k,\be_k) & = \ell_{CE}(\BPi_{\ell,\ell'}(\bar{\bz}_{\ell'} +\bb),\be_{\ell}) +  \ell_{CE}(\BPi_{\ell,\ell'}(\bar{\bz}_{\ell}+\bb),\be_{\ell'}) \\
& = \ell_{CE}(\bar{\bz}_{\ell'}+\bb,\be_{\ell'}) +  \ell_{CE}(\bar{\bz}_{\ell}+\bb,\be_{\ell})
\end{align*}
where $\BPi_{\ell,\ell'}\be_{\ell} = \be_{\ell'}$.
Also, we note that $\BPi_{\ell,\ell'}\BY = \BY$ for $\ell$ and $\ell'$ in the same cluster, and also the nuclear norm is invariant under orthogonal transform. Hence, 
\[
\| \BPi_{\ell,\ell'}\bar{\BZ}\BPi_{\ell,\ell'} \BY \|_* = \|\bar{\BZ}\BY\|_*,~~\|\BPi_{\ell,\ell'}\bb\| = \|\bb\|,
\]
and we have proven the claim. In other words, we have
\[
\blkdiag(\BPi_{\Gamma_{1}},\cdots,\BPi_{\Gamma_{J}})^{\top} \bar{\BZ} \blkdiag(\BPi_{\Gamma_{1}},\cdots,\BPi_{\Gamma_{J}}) = \bar{\BZ}
\]
where $\BPi_{\Gamma_j}$ is any $|\Gamma_j|\times |\Gamma_j|$ permutation matrix that acts on the index set $\Gamma_j.$

\paragraph{Theorem~\ref{thm:main1}(c)}
In the balanced case, $\bar{\BZ} = a(K\I_K - \BJ_K)$ holds and $\bb = 0.$
\begin{align*}
\bar{\BP}  =
\frac{e^{-a} \BJ_K+ e^{-a}(e^{aK} -1)\I_{K}}{e^{-a}(K-1 + e^{aK} )} = \frac{ \BJ_K+ (e^{aK} -1)\I_{K}}{K-1 + e^{aK} } 
\end{align*}
and 
\begin{align*}
\I_K - \bar{\BP} = \I_K - \frac{ \BJ_K+ (e^{aK} -1)\I_{K}}{K-1 + e^{aK} } = \frac{K\I_K - \BJ_K}{K-1 + e^{aK}}.
\end{align*}
Then the optimality condition~\eqref{cond:1stred} indicates
\begin{equation*}
  \begin{aligned}
\frac{K\I_K - \BJ_K}{K-1 + e^{aK}} &= N\lambda_Z \left( \left[\left(\bar{\BZ}\BD\bar{\BZ}^{\top} \right)^{\dagger}\right]^{1/2}\bar{\BZ} + \bar{\BR} \right)  \\
&= N\lambda_Z\left(\sign(a)\sqrt{\frac{K}{N}}\left(\BI_K - \BJ_K/K\right) + \bar{\BR} \right)
  \end{aligned}  
\end{equation*}
where $\BD = (N/K)\I_K.$
Suppose $a\neq0$, and then $\bar{\BZ}$ has rank $K-1$. We set $\bar{\BR}=0$ according to Corollary~\ref{Cor:Op} and the above optimality condition implies,
\[
\frac{K\I_K - \BJ_K}{K-1 + e^{aK}} =\sign(a)\sqrt{KN}\lambda_Z(\I_K - \BJ_K/K) \Longrightarrow
\frac{K}{K-1 + e^{aK}} = \sign(a)\sqrt{KN}\lambda_Z.
\]
The above equation has a solution for $a>0$, given by
\[
a = \frac{1}{K}\log{\left(\frac{\sqrt{K}}{\sqrt{N}\lambda_Z}-K+1\right)}~~\text{if}~~N\lambda_Z < \sqrt{\frac{N}{K}}.
\]
If $N\lambda_Z  \geq \sqrt{N/K}$, then we select $a=0$ and the optimality condition implies $\bar{\BR}$ satisfies
\[
\bar{\BR} = \frac{1}{N\lambda_Z} \left(\I_K - \frac{\BJ_K}{K}\right)
\]
where $\|\bar{\BR}\BD^{1/2}\| = \sqrt{N/K} \|\bar{\BR} \|\leq 1$ meets the requirement of the optimality condition. 
\paragraph{Theorem~\ref{thm:main1}(d)}
The proof directly follows from~\eqref{eq:nuc} of~\cite[Lemma 5.1]{RFP10}. More precisely, we write 
\[
\BZ = \bar{\BZ}\BY = \bar{\BZ}\BD^{1/2}\BD^{-1/2}\BY
\]
where $\BD^{-1/2}\BY$ has orthogonal rows. By performing the SVD on $\bar{\BZ}\BD^{1/2} = \bar{\BU}\bar{\BSigma}\bar{\BV}^{\top}$, we have
\[
\BZ = \bar{\BU}\bar{\BSigma}\bar{\BV}^{\top} \BD^{-1/2}\BY = \BW^{\top}\BH 
\]
Then~\eqref{eq:nuc} implies $\BW = \bar{\BSigma}^{1/2}\bar{\BU}^{\top}$, $\BH = \bar{\BSigma}^{1/2}\bar{\BV}^{\top} \BD^{-1/2}\BY$, and $\bar{\BH} = \bar{\BSigma}^{1/2}\bar{\BV}^{\top} \BD^{-1/2}.$

\subsection{Optimality condition and the solution structure}

In this section, we focus on a special case when the dataset contains $k_A$ classes with $n_A$ points in each class and $k_B$ classes with $n_B$ points. The total number of classes is $K = k_A + k_B$ and number of points is $N = k_An_A + k_Bn_B$. In particular, we assume $k_A\geq 2$ and $k_B\geq 2.$ To characterize the solution, it suffices to look into the optimality condition~\eqref{cond:1stred} in Corollary~\ref{Cor:Op} which involves  $\I_K - \bar{\BP}$ and $(\bar{\BZ}\BD\bar{\BZ}^{\top})^{\dagger/2} \bar{\BZ}$. We first provide the explicit expression for both of them and in fact, they are in the form of block-structure ${\cal B}$ in~\eqref{def:calB}.

We adopt the notation~\eqref{eq:barZb} for $\bar{\BZ}$ and $\bb$, and further write $\bar{\BZ}$ in the following form,
\begin{align*}
\bar{\BZ} & =  
\begin{bmatrix}
 (k_Aa+ k_Bc) \BC_{k_A}  & 0 \\
0 & (k_Ab+k_Bd)\BC_{k_B} 
\end{bmatrix} + 
k_Ak_B \bs\bs^{\top} 
\begin{bmatrix}
c\I_{k_A} & 0 \\
0 & b\I_{k_B}
\end{bmatrix}
 \in\RR^{K\times K} 
\end{align*}
where $\BC_{k_A}$ (and $\BC_{k_B}$) is defined in~\eqref{def:Ck} and 
\begin{equation}\label{def:Ds}
\bs := 
\begin{bmatrix}
\frac{\bone_{k_A}}{k_A} \\
- \frac{\bone_{k_B}}{k_B}
\end{bmatrix},~~\|\bs\| = \sqrt{\frac{1}{k_A} + \frac{1}{k_B}},~~
\BD = 
\begin{bmatrix}
n_A \I_{k_A} & 0 \\
0 & n_B\I_{k_B}
\end{bmatrix}
\end{equation}
A direct computation gives
\begin{equation}\label{eq:ZDZ}
\begin{aligned}
& \bar{\BZ}\BD\bar{\BZ}^{\top} =  
\begin{bmatrix}
n_A(k_Aa + k_Bc)^2 \BC_{k_A} & 0 \\
0 & n_B(k_Bd   + k_Ab)^2\BC_{k_B}
\end{bmatrix} + ( k_An_Bb^2+k_Bn_Ac^2) k_Ak_B\bs\bs^{\top}
\end{aligned}
\end{equation}
and
\begin{align*}
& \left[(\bar{\BZ}\BD\bar{\BZ}^{\top})^{\dagger}\right]^{\frac{1}{2}}\bar{\BZ}  = 
\begin{bmatrix}
\frac{\sign(k_Aa + k_Bc)}{\sqrt{n_A}}\BC_{k_A} & 0 \\
0 & \frac{\sign(k_Ab+k_Bd)}{\sqrt{n_B}}\BC_{k_B}
\end{bmatrix} \\
& \qquad +   \frac{\sign(k_An_Bb^2+k_Bn_Ac^2)k_Ak_B}{\sqrt{( k_An_Bb^2+k_Bn_Ac^2)(k_A+k_B)}}\cdot \bs\bs^{\top}
\begin{bmatrix}
c\I_{k_A} & 0 \\
0 & b\I_{k_B}
\end{bmatrix}
\end{align*}
is ${\cal B}(a_Z,b_Z,c_Z,d_Z)$ which satisfies 
\begin{equation}\label{eq:az}
\begin{aligned}
& k_Aa_Z + k_Bc_Z = n^{-1/2}_A\sign(k_Aa + k_Bc), ~~k_Ab_Z + k_Bd_Z  = n^{-1/2}_B\sign(k_Ab+k_Bd), \\
& b_Z  = \frac{b\sign(k_An_Bb^2+k_Bn_Ac^2)}{\sqrt{( k_An_Bb^2+k_Bn_Ac^2) (k_A+k_B)}},~~c_Z  = \frac{c\sign(k_An_Bb^2+k_Bn_Ac^2)}{\sqrt{( k_An_Bb^2+k_Bn_Ac^2) (k_A+k_B)}}. \\
\end{aligned}
\end{equation}

In particular, if $a = b = c=d = 1/K$, then
\[
\bar{\BZ} = \I_K - \BJ_K/K
\]
and
\begin{equation}\label{eq:eigZs}
\begin{aligned}
\bar{\BZ}\BD\bar{\BZ}^{\top}  =  
\begin{bmatrix}
n_A \BC_{k_A} & 0 \\
0 & n_B \BC_{k_B}
\end{bmatrix}  + \frac{n_A k_B +k_An_B}{k_A+k_B}\cdot \left(\frac{k_Ak_B}{k_A + k_B}\right) \bs\bs^{\top}.
\end{aligned}
\end{equation}
The eigenvalues are $n_A$ with multiplicity $k_A-1$, $n_B$ with multiplicity $k_B-1$, $( n_A k_B +k_An_B)/K$ with multiplicity 1, and 0 with multiplicity 1.

For $\I_K - \bar{\BP}$, direct computation implies that ${\cal B}(a_P,b_P,c_P,d_P)$, i.e.,
\begin{equation}\label{eq:IP}
\I_K - \bar{\BP} = 
\begin{bmatrix}
  a_Pk_A \BC_{k_A} + c_Pk_B \I_{k_A}  & -b_P\BJ_{k_A\times k_B} \\
  -c_P\BJ_{k_B\times k_A} & d_Pk_B\BC_{k_B}+ b_Pk_A\I_{k_B}
\end{bmatrix}
\end{equation}
where
\begin{equation}\label{eq:ap}
\begin{aligned}
a_P & = \frac{e^{-a+mk_B}}{k_Be^{-c-mk_A}+e^{-a+mk_B}(k_A-1 + e^{k_Aa +k_Bc} )}, \\
b_P & = \frac{e^{-b+mk_B}}{k_Ae^{-b+mk_B} + e^{-d-mk_A}(k_B-1 +e^{k_Ab+k_Bd})}, \\
c_P & = \frac{e^{-c-mk_A} }{k_Be^{-c-mk_A}+e^{-a+mk_B}(k_A-1 + e^{k_Aa +k_Bc} )}, \\
d_P & = \frac{e^{-d-mk_A}}{k_Ae^{-b+mk_B} + e^{-d-mk_A}(k_B-1 +e^{k_Ab+k_Bd})}.
\end{aligned}
\end{equation}

To characterize the solution structure w.r.t. $\lambda_Z$ and $\lambda_b$, we look into the optimality condition in Corollary~\ref{Cor:Op}:
\[
N^{-1}(\I_K-\bar{\BP}) = \lambda_Z \left( \left[(\bar{\BZ}\BD\bar{\BZ}^{\top})^{\dagger}\right]^{\frac{1}{2}}\bar{\BZ}  + \bar{\BR}\right)
\]
where  $\bar{\BR}\BD\bar{\BZ}^{\top} = 0$, $ \bar{\BZ}^{\top}\bar{\BR} = 0$, and $\|\bar{\BR}\BD^{1/2}\| \leq 1$. The key idea of the proof is straightforward: for different regimes of $\lambda_Z$, we will explicitly construct $\bar{\BR}$ and show that a solution exists for the nonlinear equation system above. Then by restricted strong convexity in Lemma~\ref{lem:varphi}, we know that this solution must be a unique global minimizer to~\eqref{eq:cvx}.

Now we will present the main result on how the solution structure changes w.r.t. the varying $\lambda_Z$. To characterize this, we first introduce a few functions that will be used later. 
Let 
\begin{equation}\label{def:f2proof}
  \begin{aligned}
 & f_2(t,\lambda_Z,\lambda_b) : = g_2(x_2(t)) - (k_A+k_B)k_Am(t,\lambda_Z,\lambda_b)\\
  & \qquad =  \log{\left[\left(\frac{\sqrt{n_A}}{N\lambda_Z}-1\right)\left(k_A+k_Bx_2(t)\right)+1\right]} - k_A\log{x_2(t)}-(k_A+k_B)k_Am(t,\lambda_Z,\lambda_b)
  \end{aligned}
\end{equation}
where $g_2$ is defined in~\eqref{def:g12}, and 
    \begin{align}
     &m(t,\lambda_Z,\lambda_b) = \frac{\lambda_Z}{\lambda_b}\frac{n_A-n_Bt}{\sqrt{( k_Bn_A +k_An_B t^2) (k_A+k_B)}},\label{def:m}\\
     &x_2(t) =  \frac{k_A}{\sqrt{(k_B+k_An_B t^2/n_A)(k_A+k_B)}-k_B}. \label{def:x2}
    \end{align}
   
The main result and proof rely on a key quantity: 
\begin{equation}\label{def:t*}
t^*(\lambda_Z,\lambda_b) = \argmin_{t \geq 0} |f_2(t,\lambda_Z,\lambda_b)| = 
\begin{cases}
\text{the root of } f_2(t,\lambda_Z,\lambda_b), &\eta(\lambda_Z) < 0, \\
0, & \eta(\lambda_Z) \geq 0,
\end{cases}
\end{equation}
where
\begin{equation}\label{def:eta}
\eta(\lambda_Z) :=  f_2(0,\lambda_Z,\lambda_b),~~~\lambda_Z \in (\sqrt{n_B}/N,\sqrt{n_A}/N).
\end{equation}
The function $\eta(\cdot)$ is decreasing and in particular $\eta(\sqrt{n_A}/N) < 0.$

Later, Lemma~\ref{lem:supp} will prove $f_2(t,\lambda_Z,\lambda_b)$ and $m(t,\lambda_Z,\lambda_b)$ are strictly increasing and decreasing respectively. Therefore, if $f_2(0,\lambda_Z,\lambda_b)\leq 0$, then $f_2$ must have a unique root and it is also equal to $t^*(\lambda_Z,\lambda_b)$; otherwise, $f_2(t,\lambda_Z,\lambda_b)$ stays positive for all $t\geq 0$ and $t=0$ is the global minimizer of $f_2$ in $t$. Moreover, we will see that $t^*(\lambda_Z,\lambda_b)$ is increasing in $\lambda_Z$ when $\lambda_b>0$ is fixed. 

For simplicity, we denote $m(t,\lambda_Z,\lambda_b)$ and $t^*(\lambda_Z,\lambda_b)$ by $m(t)$ and $t^*(\lambda_Z)$ respectively.
Now we are ready to present the main result which implies Theorem~\ref{thm:biasf} immediately.

\begin{proposition}\label{prop:bias}
  Suppose $n_A > n_B$. For $\lambda_Z$ of different regimes, the optimality condition satisfies the following properties: 
  \begin{enumerate}[(a)]
  \item If $N\lambda_Z \leq \sqrt{n_B}$, there is a unique solution $(a,b,c,d)$ to
  \[
  N^{-1}(\I_K - \bar{\BP}) = \lambda_Z \left[(\bar{\BZ}\BD\bar{\BZ}^{\top})^{\dagger}\right]^{\frac{1}{2}}\bar{\BZ} 
  \]
  that satisfies $a,b,c,d>0$ and $a-c+d-b \leq 0.$
  
  \item If $\sqrt{n_B}<N\lambda_Z<\sqrt{n_A}$ and
\begin{equation}\label{cond:caseb}
  k_B e^{-(k_A+k_B)m(t^*(\lambda_Z,\lambda_b))} < \frac{1}{N\lambda_Z}\sqrt{\left(\frac{k_Bn_A}{t^*(\lambda_Z,\lambda_b)^2}+k_An_B\right)(k_A+k_B)} - k_A,
  \end{equation}
 then
   there is a unique solution $(a,b,c,d)$ to
\[
  N^{-1}(\I_K - \bar{\BP}) = \lambda_Z \left[(\bar{\BZ}\BD\bar{\BZ}^{\top})^{\dagger}\right]^{\frac{1}{2}}\bar{\BZ} + \bar{\BR}
\]
  where
  \[
  \bar{\BR} = \frac{1}{\sqrt{n_B}}
  \begin{bmatrix}
  0 & 0 \\
  0 & \I_{k_B} - \BJ_{k_B}/k_B
  \end{bmatrix}.
  \]
  Moreover, $(a,b,c,d)$ satisfies $k_Ab + k_Bd = 0$ and $a,b,c>0$.

  \item 
If $\sqrt{n_B}<N\lambda_Z<\sqrt{n_A}$ and
\begin{equation}\label{cond:casec}
  k_B e^{-(k_A+k_B)m(t^*(\lambda_Z,\lambda_b))} > \frac{1}{N\lambda_Z}\sqrt{\left(\frac{k_Bn_A}{t^*(\lambda_Z,\lambda_b)^2}+k_An_B\right)(k_A+k_B)} - k_A,
\end{equation}
then there is a unique solution $(a,b,c,d,t,\tau)$ to
  \[
  N^{-1}(\I_K - \bar{\BP}) = \lambda_Z \left[(\bar{\BZ}\BD\bar{\BZ}^{\top})^{\dagger}\right]^{\frac{1}{2}}\bar{\BZ} + \bar{\BR}
  \]
  where
\begin{equation*}
  \begin{aligned}
    \bar{\BR} & = \begin{bmatrix}
    0 & 0 \\
    0 &\frac{1}{N\lambda_Z} (\BI_{k_B}-\BJ_{k_B}/k_B)
    \end{bmatrix}  + 
\frac{k_Ak_B\tau}{\sqrt{(k_Bn_A+k_An_Bt^2)(k_A+k_B)}} 
\bs\bs^{\top}
\begin{bmatrix}
\I_{k_A} & 0 \\
0 & t\I_{k_B}
\end{bmatrix} 
  \end{aligned}
\end{equation*}
and $\bs$ is defined in~\eqref{def:Ds}.
  Moreover, $(a,b,c,d)$ satisfies $b=c=d=0$ and $a>0$.
  
  \item If $N\lambda_Z > \max\{\sqrt{n_A},\sqrt{n_B}\}$, then $\BZ = 0$.
  
  \item If $\lambda_b = \infty$, i.e., the bias-free scenario, then~\eqref{cond:caseb} and~\eqref{cond:casec} are equivalent to $\lambda_Z <\lambda^*$ and $\lambda_Z>\lambda^*$ respectively for some $\lambda^*\in (\sqrt{n_B}/N,\sqrt{n_A}/N)$, where $\lambda^*$ satisfies
  \[
  t^*(\lambda^*) = \sqrt{ \frac{k_Bn_A}{(N\lambda_Z)^2(k_A+k_B) - k_An_B}}.
  \]
  \end{enumerate}

\end{proposition}

The intuition behind the change of solution structure $\bar{\BZ}$ is essentially the singular value thresholding. By the $2 \times 2$ block structure of $\bar{\BZ}$, $\bar{\BZ}$ only has at most $3$ different nonzero singular values regardless of the multiplicity. As the nuclear norm penalty parameter $\lambda_Z$ grows, these singular values will gradually diminish to $0$ until all the singular values become $0$ after a certain finite threshold. 
The proof relies on the following two important lemmas.
\begin{lemma}\label{lem:func}
Suppose $f_1$ and $f_2$ are continuous, and strictly decreasing and increasing respectively on $t>0$. In addition, they satisfy
\begin{align*}
\lim_{t\rightarrow 0}f_1(t) = + \infty,~~\lim_{t\rightarrow \infty}f_2(t) = +\infty
\end{align*}
and $I^+ = \{t: f_1 > 0, f_2>0\}$ is nonempty, then
\[
f(t) = \frac{f_1(t)}{f_2(t)} 
\]
is decreasing on $I^+$ and $f(t) = t$ has a unique solution.
\end{lemma}
\begin{proof}[\bf Proof of Lemma~\ref{lem:func}]
Let $t_1 = \sup\{t: f_1(t) > 0\}$ and $t_2 = \inf \{t: f_2(t)>0\}$.
Since $I^+$ is nonempty, we have $t_2 < t_1.$
\begin{itemize}
\item if $t_1 < +\infty$, then $f_1(t_1) = 0$ and $\lim_{t\rightarrow t_1}f_2(t)$ exists and is positive, and thus
$\lim_{t\rightarrow t_1^-} f(t) = 0.$

\item if $t_1 = +\infty$, then $\lim_{t\rightarrow\infty}f_1(t)$ exists and is positive and $\lim_{t\rightarrow \infty}f_2(t) = \infty$, and thus $\lim_{t\rightarrow t_1^-} f(t) = 0$.

\item if $t_2 = 0$, then $\lim_{t\rightarrow 0}f_2(t)$ exists and is positive and $\lim_{t\rightarrow 0} f_1(t) = \infty$, and thus $\lim_{t\rightarrow t_2^+}f(t) = \infty$.

\item if $t_2 > 0$, then $f_2(t_2) = 0$ and $\lim_{t\rightarrow t_2} f_1(t)$ exists and is positive, and thus $\lim_{t\rightarrow t_2^+}f(t) = \infty$.

\end{itemize}
This implies that
\[
\lim_{t\rightarrow t_1^-}f(t) = 0,~~\lim_{t\rightarrow t_2^+}f(t) = \infty.
\]
As a result, $f(t)$ is decreasing on $(t_2,t_1)$ and the equation $f(t) = t$ has a unique solution by continuity.
\end{proof}

Next, we introduce a few functions and look into their properties. These properties will be very useful in proving the existence of a solution to the optimality system under different regimes of $\lambda_Z.$

\begin{lemma}\label{lem:supp}

\begin{enumerate}[(a)]
\item  Let 
\begin{equation}\label{def:g12}
\begin{aligned}
g_1(x) := \log{\left[\left(\frac{\sqrt{n_B}}{N\lambda_Z}-1\right)\left(k_Ax+k_B\right)+1 \right]}-k_B\log{x},~~\lambda_Z\leq \sqrt{n_B}/N, \\
g_2(x) := \log{\left[\left(\frac{\sqrt{n_A}}{N\lambda_Z}-1\right)\left(k_A+k_Bx\right)+1 \right]}-k_A \log{x},~~\lambda_Z\leq \sqrt{n_A}/N,
\end{aligned}
\end{equation}
where $x>0$. 
Then $g_1$ and $g_2$ are strictly decreasing in $x$.

\item The function $f_2(t,\lambda_Z)$ in~\eqref{def:f2proof} and $m(t,\lambda_Z)$ in~\eqref{def:m} are strictly increasing and decreasing in $t$ respectively. Moreover, $t^*(\lambda_Z,\lambda_b)$ is increasing in $0 < \lambda_Z \leq \sqrt{n_A}/N$ and $t^*(\lambda_Z,\lambda_b)\leq n_A/n_B$ holds.

\item For $N\lambda_Z \in [\sqrt{n_B},\sqrt{n_A}]$, define
\begin{equation}\label{def:xi}
\xi(\lambda_Z,\lambda_b) := k_B e^{-(k_A+k_B)m(t^*(\lambda_Z,\lambda_b))} - \left( \frac{1}{N\lambda_Z}\sqrt{\left(\frac{k_Bn_A}{t^*(\lambda_Z,\lambda_b)^2}+k_An_B\right)(k_A+k_B)} - k_A\right).
\end{equation}
Then for any $\lambda_b > 0$, it holds
\begin{equation}\label{eq:range}
\xi\left(\frac{\sqrt{n_A}}{N},\lambda_b\right) > 0,~~~~~\xi\left(\frac{\sqrt{n_B}}{N},\lambda_b\right) < 0.
\end{equation}
This implies by continuity that there exists an $\eps>0$, such that:
\[
\xi\left(\frac{\sqrt{n_A}}{N} - \delta',\lambda_b\right) > 0,~~~~~\xi\left(\frac{\sqrt{n_B}}{N} + \delta'',\lambda_b\right) < 0,~~~\forall \delta',\delta'' <\eps.
\]

\item If $\lambda_b = \infty$, i.e., for the bias-free case, then 
\begin{equation}\label{def:xi2}
\xi(\lambda_Z,\infty) = (k_A+k_B) - \frac{1}{N\lambda_Z}\sqrt{\left(\frac{k_Bn_A}{t^*(\lambda_Z)^2}+k_An_B\right)(k_A+k_B)} 
\end{equation}
and it is increasing in $\lambda_Z\in [\sqrt{n_B}/N,\sqrt{n_A}/N]$. Combining with results from~\eqref{eq:range}, we have there exists a unique $\lambda^* \in (\sqrt{n_B}/N,\sqrt{n_A}/N)$ such that
$$
\xi(\lambda^*,\infty) = 0.
$$

\end{enumerate}
\end{lemma}

\begin{proof}[\bf Proof of Lemma~\ref{lem:supp}]
\noindent(a) Consider the first-order derivative of $g_1(x)$, we have
\[
g_1'(x) = \frac{k_A\left(\frac{\sqrt{n_B}}{N\lambda_Z}-1\right)}{\left(\frac{\sqrt{n_B}}{N\lambda_Z}-1\right)\left(k_Ax+k_B\right)+1} - \frac{k_B}{x} < \frac{1}{x} - \frac{k_B}{x}\leq 0 
\]
provided that $x>0$ and $N\lambda_Z<\sqrt{n_B}$, and similarly $g_2'(x)<0$ holds for $x>0$ if $N\lambda_Z < \sqrt{n_A}.$
\\
\\
\noindent (b) 
We first note that $f_2(t,\lambda_Z,\lambda_b)  = g_2(x_2(t)) - (k_A+k_B)k_Am(t).$
Since $x_2(t)$ in~\eqref{def:x2} is strictly decreasing w.r.t. $t$ in its domain and so is $g_2$, as shown in Lemma~\ref{lem:supp}(a), we conclude $g_2(x_2(t))$ is strictly increasing. 
For $m(t)$, we notice that
\begin{equation}\label{eq:mderi}
\begin{aligned}
m'(t) &= -\frac{\lambda_Z}{\lambda_b}\frac{n_B(k_Bn_A+n_Bk_At^2)(k_A+k_B)+k_An_Bt(n_A-n_Bt)(k_A+k_B)}{\left[(k_Bn_A+n_Bk_At^2)(k_A+k_B) \right]^{\frac{3}{2}}}\\
&=-\frac{\lambda_Z}{\lambda_b}n_Bn_A(k_A+k_B) \frac{k_B+k_At}{\left[(k_Bn_A+n_Bk_At^2)(k_A+k_B) \right]^{\frac{3}{2}}} < 0
\end{aligned}
\end{equation}
This implies $m$ is decreasing in $t$, and then $f_2(t)$ is increasing in $t$. This proves the first half of the argument.

Now we investigate the monotonicity of $t^*$ w.r.t. $\lambda_Z.$ The idea is to apply implicit differentiation on $f_2$ at $t^*(\lambda_Z,\lambda_b)$. We have 
\begin{equation}\label{eq:impdiff}
\frac{\pa f_2}{\pa t}\cdot \frac{\diff t^*(\lambda_Z,\lambda_b)}{\diff \lambda_Z} + \frac{\pa f_2}{\pa \lambda_Z} = 0 \Longrightarrow \frac{\diff t^*(\lambda_Z,\lambda_b)}{\diff \lambda_Z} = -\frac{\pa f_2/\pa \lambda_Z}{\pa f_2 /\pa t}
\end{equation}
where we already know $\partial f_2/\partial t > 0$ from (b). Therefore it suffices to check the monotonicity of $f_2$ in $\lambda_Z$ at $t^*(\lambda_Z,\lambda_b)$. 
Notice for any fixed $t < n_A/n_B~ (m(t)>0)$ and $\lambda_b>0$, we have that $f_2(t,\lambda_Z,\lambda_b)$ is decreasing in $\lambda_Z$. Finally, we claim $t^*(\lambda_Z,\lambda_b) < n_A/n_B$, so at $t^*(\lambda_Z,\lambda_b)$, $\partial f_2/\partial \lambda_Z <0$. Therefore by~\eqref{eq:impdiff}, we have $\partial t^*(\lambda_Z,\lambda_b) / \partial \lambda_Z > 0$.

Now we are left with proving the claim $t^*(\lambda_Z,\lambda_b) < n_A/n_B$, which follows from the following simple argument. By plugging $t = n_A/n_B$ into $f_2$, it holds that
\[
m(n_A/n_B) = 0,~~~x_2(n_A/n_B) < x_2(\sqrt{n_A/n_B}) = 1
\]
where $n_A/n_B>1$ and $x_2$ is decreasing. Therefore, by the fact $g_2$ is an increasing function of $x$, we have
$$
f_2\left(\frac{n_A}{n_B}\right) = g_2\left(x_2\left(\frac{n_A}{n_B}\right)\right) > g_2\left(x_2\left(\sqrt{\frac{n_A}{n_B}}\right)\right) = g_2(1) >0,
$$
and by the monotonicity of $f_2$, $t^*(\lambda_Z,\lambda_b)$ must be smaller than $n_A/n_B.$ 
\\
\\
\noindent(c) We evaluate the value of $\xi$ at $\lambda_Z = \sqrt{n_A}/N$ and $\lambda_Z = \sqrt{n_B}/N$ separately.
\\
\\
\textbf{Right endpoint: at $\lambda_Z = \sqrt{n_A}/N$}. Note that
\begin{align*}
f_2(\sqrt{n_A/n_B}, \sqrt{n_A}/N,\lambda_b) & = -(k_A+k_B)k_Am(\sqrt{n_A/n_B}, \sqrt{n_A}/N,\lambda_b) < 0, \\
f_2(n_A/n_B, \sqrt{n_A}/N,\lambda_b) & > 0,
\end{align*}
where $x_2(\sqrt{n_A/n_B}) = 1$. This means a root exists within the interval $(\sqrt{n_A/n_B},n_A/n_B)$ for $f_2$ with $\lambda_Z = \sqrt{n_A}/N$. Denote $t^* = t^*(\lambda_Z,\lambda_b)$, $x_2^* = x_2(t^*(\lambda_Z,\lambda_b))$, and $\xi(\lambda_Z) = \xi(\lambda_Z,\lambda_b)$.
By definition, the root $t^*$ satisfies
\begin{align}\label{eq:x2eq}
& x_2(t^*) = e^{-(k_A+k_B)m(t^*)} = \frac{k_A}{\sqrt{(k_B+n_B k_At^{*2} /n_A)(k_A+k_B)}-k_B} \nonumber \\
& \Longleftrightarrow t^*\sqrt{\left(\frac{k_B}{t^{*2}}+\frac{n_B  k_A}{n_A}\right)(k_A+k_B)} = k_A e^{(k_A+k_B)m(t^*)} + k_B = \frac{k_A}{x_2^*} + k_B. 
\end{align}
Then applying~\eqref{eq:x2eq} gives
\begin{align*}
\xi\left(\frac{\sqrt{n_A}}{N}\right) & = k_B e^{-(k_A+k_B)m(t^*)}  - \left( \sqrt{\left(\frac{k_B}{t^{*2}}+\frac{k_An_B}{n_A}\right)(k_A+k_B)} - k_A\right) \\
& =  k_A+k_B e^{-(k_A+k_B)m(t^*)} - \frac{1}{t^*} (k_A e^{(k_A+k_B)m(t^*)} + k_B) \\
& = \left( k_A+k_B e^{-(k_A+k_B)m(t^*)} \right) \left( 1 - \frac{e^{(k_A+k_B)m(t^*)}}{t^*}\right).
\end{align*}
Note that $1<\sqrt{n_A/n_B} < t^*< n_A/n_B$:
\begin{align*}
e^{(k_A+k_B)m(t^*)} & = \frac{\sqrt{(k_B+n_B t^{*2} k_A/n_A)(k_A+k_B)}-k_B}{k_A}  \leq \frac{1}{2} \left(  \frac{n_B}{n_A} t^{*2}  + 1  \right) <\frac{1}{2}(t^*+1) < t^*
\end{align*}
where for the first inequality, we use $\sqrt{xy} \leq (x+y)/2$ for $x,y>0$ and for next two, we use $t^*<n_A/n_B$ and $t^*>1$ respectively. Therefore, $\xi\left(\sqrt{n_A}/N\right) > 0$ holds.
\\
\\
\textbf{Left endpoint: at $\lambda_Z = \sqrt{n_B}/N$}. There are two possible cases. If $f_2(0,\sqrt{n_B}/N,\lambda_b)\geq 0$, then $t^* = 0$ and $\xi(\sqrt{n_B}/N) = -\infty$ holds automatically. Otherwise, $t^* > 0$ and it satisfies
\begin{align*}
& 0<\frac{1}{k_A}\log{\left[\left(\sqrt{\frac{n_A}{n_B}}-1\right)\left(k_A+k_Bx_2^* \right)+1\right]} = \log{x_2^*} + (k_A+k_B)m(t^*) \Longrightarrow e^{-(k_A+k_B)m(t^*)} < x_2^*.
\end{align*}
Then it holds that
\begin{align*}
& \xi\left(\frac{\sqrt{n_B}}{N}\right)  = k_B e^{-(k_A+k_B)m(t^*)} - \left( \sqrt{ \frac{n_A}{n_B}} \sqrt{\left(\frac{k_B}{t^{*2}}+\frac{k_An_B}{n_A}\right)(k_A+k_B)} - k_A\right) \\
& < k_A +  k_Bx_2^* -  \sqrt{ \frac{n_A}{n_B}}\cdot \frac{1}{t^*} \left( \frac{k_A}{x_2^*} +k_B\right) \\
& \leq (k_A +  k_B  x_2^*) \left(1 - \sqrt{ \frac{n_A}{n_B}}\cdot \frac{1}{t^*x_2^*}\right) \leq 0
\end{align*}
where the second line follows from~\eqref{eq:x2eq}, and the last line uses
\[
tx_2(t) = \frac{k_A t}{\sqrt{(k_B+k_An_B t^2/n_A)(k_A+k_B)}-k_B} \leq \frac{k_At}{k_A\sqrt{n_B/n_A}t} \leq \sqrt{\frac{n_A}{n_B}},~~\forall t>0,
\]
since 
\[
\sqrt{\left(k_B+\frac{k_An_Bt^2}{n_A} \right)(k_A+k_B)}-k_B \geq k_B \left(\sqrt{\left(1+\frac{k_An_Bt^2}{k_Bn_A}\right)\left(1+\frac{k_A}{k_B}\right)}-1\right) \geq k_A \sqrt{\frac{n_B}{n_A}}t
\]
where the last inequality follows from $(1+x)(1+y) \geq (1+\sqrt{xy})^2$ for $x,y \geq 0$.
\\
\\
\noindent(d) For any $\lambda_Z \in [\sqrt{n_B}/N,\sqrt{n_A}/N]$ and $t\in [0,n_A/n_B]$, we know that $m(t,\lambda_Z,\lambda_b)$ is uniformly bounded. As $\lambda_b\rightarrow\infty$, we have
\[
\lim_{\lambda_b\rightarrow\infty} \sup_{0\leq t\leq n_A/n_B, \sqrt{n_B}\leq N\lambda_Z\leq \sqrt{n_A}} |m(t,\lambda_Z,\lambda_b)| = 0.
\]
As a result, as $\lambda_b\rightarrow\infty$, 
\[
\lim_{\lambda_b\rightarrow\infty} \xi(\lambda_Z,\lambda_b) = (k_A+k_B) - \frac{1}{N\lambda_Z}\sqrt{\left(\frac{k_Bn_A}{t^*(\lambda_Z)^2}+k_An_B\right)(k_A+k_B)}.
\]
Note that $t^*(\lambda_Z)$ is increasing in $\lambda_Z$ (proven in (b)), and so is $\xi(\lambda_Z,\infty)$. From (c), we know that there exists a $\lambda^*$ satisfying $\sqrt{n_B}\leq N\lambda^*\leq \sqrt{n_A}$ such that for any $\lambda_Z< \lambda^*$, $\xi(\lambda_Z,\infty) < 0$ holds, and for any $\lambda_Z>\lambda^*$, $\xi(\lambda_Z,\infty) > 0$ holds. This $\lambda^*$ is the zero to $\xi(\lambda_Z,\infty)$, i.e.,
\[
t^*(\lambda^*) = \sqrt{ \frac{k_Bn_A}{(N\lambda_Z)^2(k_A+k_B) - k_An_B}}.
\]
\end{proof}

\subsection{Proof of Proposition~\ref{prop:bias}}
The idea is straightforward, according to our parametrization of $\bar{\BZ}$ and $\bb$ in~\eqref{eq:barZb}, we have $\bone^{\top}(\bar{\BZ},\bb)=0$, which is necessary for being the global minimizer by Lemma~\ref{lem:opt}. 
By strong convexity on the set $\{(\bar{\BZ},\bb)|\bone^{\top}~(\bar{\BZ}+\bb\bone_N^{\top})=0 \}$ given by Lemma~\ref{lem:varphi}, it suffices to verify the first order optimality equation has a solution in each regime of $\lambda_Z$ which is automatically unique and global by strong convexity. 
Therefore the key ideas of the proof are similar for all four cases. We first propose a candidate solution for $\bar{\BZ}$ and certificate $\bar{\BR}$ such that they are admissible according to Corollary~\ref{Cor:Op}, i.e.,
\begin{equation}\label{eq:canstur}
  \bar{\BR}\BD\bar{\BZ}^{\top} = 0,~~ \bar{\BZ}^{\top} \bar{\BR} = 0, ~~\|\bar{\BR}\BD^{1/2} \| \leq 1.
\end{equation} 
Then we derive the corresponding optimality equation system and prove the existence of a solution. In the following proof, we will let $\lambda_{b} > 0$ be any positive number and the bias vector $\bb$ is in the form of~\eqref{eq:barZb}.
\begin{proof}[\bf Proof of Proposition~\ref{prop:bias}(a)]We consider $N\lambda_Z < \sqrt{n_B}$.\\
\textbf{Solution structure: }In this case, we propose the structure:
\begin{align*}
  \bar{\BZ} = \mathcal{B}(a,b,c,d):
      & = \begin{bmatrix}
     a(k_A\I_{k_A} -\BJ_{k_A\times k_A}) + k_Bc \I_{k_A}  & -b\BJ_{k_A\times k_B} \\
     -c\BJ_{k_B\times k_A} & d(k_B\BI_{k_B}-\BJ_{k_B \times k_B}) + k_Ab\BI_{k_B}
     \end{bmatrix}, ~~~\bar{\BR} = 0.
\end{align*}\\
Notice conditions~\eqref{eq:canstur} are trivially satisfied as $\bar{\BR}$ being zero.\\
\\
\textbf{Optimality system:} Note that $\I_K -\bar{\BP}$ and $\left[(\bar{\BZ}\BD\bar{\BZ}^{\top})^{\dagger}\right]^{\frac{1}{2}}\bar{\BZ} $ are in the block structure with ${\cal B}(a_P,b_P,c_P,d_P)$ in~\eqref{eq:ap} and ${\cal B}(a_Z,b_Z,c_Z,d_Z)$ in~\eqref{eq:az}. 
By comparing the coefficients, we need to ensure
\begin{equation*}
\begin{aligned}
k_Aa_P+k_Bc_P &= N\lambda_Z(k_Aa_Z+k_Bc_Z),~~k_Ab_P+k_Bd_P = N\lambda_Z(k_Ab_Z+k_Bd_Z),\\
~~b_P &=N\lambda_Z b_Z > 0,~~c_P = N\lambda_Z c_Z > 0.
\end{aligned}
\end{equation*}
Then the equations above along with the second equation in~\eqref{cond:1stred}, $N^{-1}(\BI_K-\bar{\BP})\bn = \lambda_b\bb$, yield to the following system,
\begin{equation}\label{EQ:OP:Original}
\begin{aligned}
& \frac{k_Ae^{-a+mk_B}+k_Be^{-c-mk_A}}{k_Be^{-c-mk_A}+e^{-a+mk_B}(k_A-1 + e^{k_Aa +k_Bc} )} = \frac{N\lambda_Z}{\sqrt{n_A}}, \\
& \frac{ k_Ae^{-b+mk_B} + k_B e^{-d-mk_A}}{k_Ae^{-b+mk_B} + e^{-d-mk_A}(k_B-1 +e^{k_Ab+k_Bd})} = \frac{N\lambda_Z}{\sqrt{n_B}}, \\
& \frac{e^{-c-mk_A}}{k_Be^{-c-mk_A}+e^{-a+mk_B}(k_A-1 + e^{k_Aa +k_Bc} )} = \frac{N\lambda_Z c}{\sqrt{( k_An_Bb^2+k_Bn_Ac^2) (k_A+k_B)}}, \\
& \frac{e^{-b+mk_B}}{k_Ae^{-b+mk_B} + e^{-d-mk_A}(k_B-1 +e^{k_Ab+k_Bd})} =   \frac{N\lambda_Z b}{\sqrt{( k_An_Bb^2+k_Bn_Ac^2) (k_A+k_B)}}, \\
& m = \frac{\lambda_Z}{\lambda_b}\frac{n_Ac-n_Bb}{\sqrt{(k_An_Bb^2+k_Bn_Ac^2) (k_A+k_B)}}.
\end{aligned}
\end{equation}
Our goal now is to show that there exists a solution to~\eqref{EQ:OP:Original} with $b > 0$ and $c>0.$
Let $t = b/c$, and then combining the four equations in~\eqref{EQ:OP:Original} gives
\begin{equation}\label{EQ:Reduction:OPW}
\begin{aligned}
\text{eq.1 in}~\eqref{EQ:OP:Original}~~~ &c = \frac{g_2(e^{a-c-m(k_A+k_B)})}{k_A+k_B} - mk_A, \\
\text{eq.2 in}~\eqref{EQ:OP:Original}~~~ &b = \frac{g_1(e^{d-b+m(k_A+k_B)})}{k_A+k_B} + mk_B, \\
\text{eq.1 and 3 in}~\eqref{EQ:OP:Original}~~~ & x_2(t):=e^{a-c-m(k_A+k_B)} = \frac{k_A}{\sqrt{(k_B+k_An_B t^2/n_A)(k_A+k_B)}-k_B}, \\
\text{eq.2 and 4 in}~\eqref{EQ:OP:Original}~~~& x_1(t):=e^{d-b+m(k_A+k_B)} = \frac{k_B}{\sqrt{(k_A+n_A k_B t^{-2} /n_B)(k_A+k_B)}-k_A}, \\
\text{eq.5 in}\eqref{EQ:OP:Original}~~~&m = \frac{\lambda_Z}{\lambda_b}\frac{n_A-n_Bt}{\sqrt{(k_Bn_A +k_An_Bt^2) (k_A+k_B)}} 
\end{aligned}
\end{equation}
where $g_1$ and $g_2$ are defined in~\eqref{def:g12}.
Here 
\begin{equation}\label{def:x12case1}
x_1(t) := e^{d-b+m(k_A+k_B)},~~~x_2(t) := e^{a-c-m(k_A+k_B)}
\end{equation}
respectively, and $x_1(t)$ and $x_2(t)$ are monotonic increasing and decreasing function of $t$ on $(t>0)$ respectively in~\eqref{EQ:Reduction:OPW}. Using the fact $(1+x)(1+y) \geq (1+\sqrt{xy})^2$ for any $x\geq 0$ and $y\geq 0$ leads to
\begin{equation}\label{def:x12case1est}
x_1(t) \leq \frac{k_B}{k_A}\sqrt{\frac{n_B}{n_A}}t,~~x_2(t)\leq \frac{k_A}{k_B}\sqrt{\frac{n_A}{n_B}}t^{-1}.
\end{equation}

Finally, we simplify the equations~\eqref{EQ:Reduction:OPW} into 
\begin{equation}
\begin{aligned}
b &= \frac{g_1(x_1(t))}{k_A+k_B} + mk_B, ~~c = \frac{g_2(x_2(t))}{k_A+k_B} - mk_A.
\end{aligned}
\end{equation}
Now, the goal is to show the existence of solution set $(a,b,c,d,m)$ for system~\eqref{EQ:Reduction:OPW}.\\
\\
\textbf{Existence of a solution to~\eqref{EQ:Reduction:OPW}:} Observe that the existence of $(a,d,m)$ fully relies on the existence of $(b,c)$ through eq.3-5 in~\eqref{EQ:Reduction:OPW}. Notice, the RHS of both eq.$1$ and $2$ only depends on $t:=b/c$, so the existence of $(b,c)$ is equivalent to whether the following single-variable equation has a solution for $t>0$:
\begin{equation}\label{EQ:t}
t:=\frac{b}{c} = \frac{g_1(x_1(t))+(k_A+k_B)k_Bm}{g_2(x_2(t))-(k_A+k_B)k_Am}=:\frac{f_1(t)}{f_2(t)}
\end{equation}
It remains to show~\eqref{EQ:t} has a solution. The argument will rely on the monotonicity of $f_1$ and $f_2$, and Lemma~\ref{lem:func}.

Note that $g_1,g_2$ and $m$ are decreasing in $t$, as shown in Lemma~\ref{lem:supp}.
Therefore, $f_1(t)$ and $f_2(t)$ are monotonically decreasing and increasing respectively on $(0,\infty)$. 
Note that
\[
\underset{t \to 0^+}{\lim} m(t) = \frac{\lambda_Z}{\lambda_b}\sqrt{\frac{n_A}{k_B(k_A+k_B)}},\quad \underset{t \to +\infty}{\lim} m(t) = -\frac{\lambda_Z}{\lambda_b}\sqrt{\frac{n_B}{k_A(k_A+k_B)}}.
\]
and
\[
\lim_{t\rightarrow 0^+}f_1(t) = \infty,~~~\lim_{t\rightarrow\infty} f_2(t) = \infty.
\]

Lemma~\ref{lem:func} indicates that it remains to see if $\{t: f_1(t)> 0,f_2(t)>0\}$ is empty. 
Due to the monotonicity of $f_1$ and $f_2$, it suffices to show that it is impossible to have $f_1(t) < 0$ and $f_2(t) < 0$ for some $t$. We will prove this is impossible by contradiction. Now, we assume $k_1f_1(t) + k_B f_2(t) \leq 0$ holds.
Note that
\begin{align*}
& k_1f_1(t) + k_B f_2(t) = k_1 g_1(x_1(t)) + k_2g_2(x_2(t)) \leq 0  \\
& \qquad \Longleftrightarrow  k_A \log\left[\left(\frac{\sqrt{n_B}}{N\lambda_Z}-1\right) \left(k_Ax_1(t)+k_B \right) +1\right] \\
& \qquad \qquad + k_B\log\left[\left(\frac{\sqrt{n_A}}{N\lambda_Z}-1\right) \left(k_B x_2(t) + k_A\right)+1 \right] \leq k_Ak_B \log x_1(t)x_2(t)
\end{align*}
where $x_1(t)x_2(t) \leq 1$ follow from~\eqref{def:x12case1est}.
This leads to a contradiction, as the left-hand side is strictly positive. Therefore, $k_1f_1(t) + k_B f_2(t) > 0$ holds which implies that $\{ t:f_1(t) > 0,f_2(t)>0 \}$ is nonempty. Using Lemma~\ref{lem:func} finishes the proof, as it implies a root exists for~\eqref{EQ:t}.
\end{proof}

\begin{proof}[\bf Proof of Proposition~\ref{prop:bias}(b)]
We consider $\sqrt{n_B} < N\lambda_Z < \sqrt{n_A}$ and $\lambda_Z$ satisfies~\eqref{cond:caseb}. 
\textbf{Solution structure: }In this case, we propose the structure
\begin{equation}\label{EQ:struc2}
  \bar{ \BZ} 
   = \begin{bmatrix}
   a(k_A\I_{k_A} -\BJ_{k_A\times k_A}) + k_Bc \I_{k_A}  & -b\BJ_{k_A\times k_B} \\
   -c\BJ_{k_B\times k_A} &-d \BJ_{k_B}
   \end{bmatrix},   ~~~\bar{\BR} = \begin{bmatrix}
     0 & 0 \\
    0 & \alpha(\BI_{k_B}-\BJ_{k_B}/k_B)
   \end{bmatrix}
 \end{equation}
In other words, $\bar{\BZ} = {\cal B}(a,b,c,d)$ with $k_Ab+k_Bd=0$ and $\bar{\BR}$ is ${\cal B}(0,0,0,\alpha/k_B)$. We have $\bar{\BR}\BD\bar{\BZ}^{\top} = \bar{\BZ}^{\top}\bar{\BR}=0$ and $\|\bar{\BR}\BD^{1/2}\|\leq 1$ holds as long as $\alpha \sqrt{n_B} \leq 1$.\\
\\
\textbf{Optimality system: }Based on~\eqref{EQ:struc2}, $[(\bar{\BZ}\BD\bar{\BZ})^{\dagger}]^{1/2}\bar{\BZ}$ in~\eqref{eq:az} reduces to ${\cal B}(a_Z,b_Z,c_Z,d_Z)$ satisfying
\begin{equation}\label{eq:az1}
\begin{aligned}
& k_Aa_Z + k_Bc_Z = n^{-1/2}_A\sign(k_Aa + k_Bc), ~~~k_Ab_Z + k_Bd_Z = 0, \\
& b_Z = \frac{b\sign(k_An_Bb^2+k_Bn_Ac^2)}{\sqrt{( k_An_Bb^2+k_Bn_Ac^2) (k_A+k_B)}}, ~~~c_Z = \frac{c\sign(k_An_Bb^2+k_Bn_Ac^2)}{\sqrt{( k_An_Bb^2+k_Bn_Ac^2) (k_A+k_B)}}, \\
\end{aligned}
\end{equation}
and $\I_K - \bar{\BP}$ in~\eqref{eq:ap} becomes ${\cal B}(a_P,b_P,c_P,d_P)$ with
\begin{equation}\label{eq:ap1}
\begin{aligned}
a_P & = \frac{e^{-a+mk_B}}{k_Be^{-c-mk_A}+e^{-a+mk_B}(k_A-1 + e^{k_Aa +k_Bc} )}, \\
b_P & = \frac{e^{-b+mk_B}}{k_Ae^{-b+mk_B} + k_Be^{-d-mk_A}}, \\
c_P & = \frac{e^{-c-mk_A} }{k_Be^{-c-mk_A}+e^{-a+mk_B}(k_A-1 + e^{k_Aa +k_Bc} )}, \\
d_P & = \frac{e^{-d-mk_A}}{k_Ae^{-b+mk_B} + k_Be^{-d-mk_A}}.
\end{aligned}
\end{equation}

Using the optimality condition $\I_K - \bar{\BP} = N\lambda_Z( (\bar{\BZ}\BD\bar{\BZ}^{\top})^{\dagger/2}\bar{\BZ} + \bar{\BR})$ leads to a normal equation similar to \eqref{EQ:OP:Original}:
\begin{equation*}
  \begin{aligned}
  k_Aa_P+k_Bc_P &= N\lambda_Z(k_Aa_Z+k_Bc_Z), \\
  k_Ab_P+k_Bd_P & = 1 = N\lambda_Z(k_A b_Z + k_B d_Z + \alpha) = N\lambda_Z\alpha,\\
 b_P &=N\lambda_Z b_Z > 0, \\
 c_P &= N\lambda_Z c_Z > 0,
  \end{aligned}
\end{equation*}
where $(\bar{\BZ}\BD\bar{\BZ}^{\top})^{\dagger/2}\bar{\BZ} + \bar{\BR}$ is ${\cal B}(a_Z,b_Z,c_Z,d_Z+\alpha/k_B)$ and $k_A b_Z + k_B d_Z=0$.
By comparing the coefficients, we have
\begin{equation}\label{eq:optbf2}
\begin{aligned}
& \frac{k_Ae^{-a+mk_B}+k_Be^{-c-mk_A}}{k_Be^{-c-mk_A}+e^{mk_A-a}(k_A-1 + e^{k_Aa +k_Bc} )} = \frac{N\lambda_Z}{\sqrt{n_A}}, \\
& \frac{e^{-b+mk_B}}{k_Ae^{-b+mk_B} + k_Be^{-d-mk_A}} =   \frac{N\lambda_Z b}{\sqrt{( k_An_Bb^2+k_Bn_Ac^2) (k_A+k_B)}}, \\
& \frac{e^{-c-mk_A}}{k_Be^{-c-mk_A}+e^{-a+mk_B}(k_A-1 + e^{k_Aa +k_Bc} )} = \frac{N\lambda_Z c}{\sqrt{( k_An_Bb^2+k_Bn_Ac^2) (k_A+k_B)}}, \\
& k_Ab+k_Bd=0,~~1 = N\lambda_Z \alpha,~\\
& m = \frac{\lambda_Z}{\lambda_b}\frac{n_Ac-n_Bb}{\sqrt{( k_An_Bb^2+k_Bn_Ac^2) (k_A+k_B)}}.
\end{aligned}
\end{equation}
Our goal is to show that there exists a solution $(a,b,c,-k_Ab/k_B)$ and $\alpha$ such that the equations above have a unique solution. 
Here $\left\|\bar{\BR} \BD^{1/2}\right\| = \sqrt{n_B}\alpha = \sqrt{n_B}/N\lambda_Z \leq 1$ with $\alpha = 1/N\lambda_Z$ if $\lambda_Z \geq \sqrt{n_B}/N$, so the certificate $\bar{\BR}$ is feasible. By letting $t = b/c$, 
we just need to show the following simplified system derived from~\eqref{eq:optbf2} has a solution for $(a,b,c,d,m)$:
\begin{equation}\label{eq:bf2ops}
  \begin{aligned}
 \text{eq.1 in}~\eqref{eq:optbf2}~~ & c = \frac{g_2(e^{a-c-m(k_A+k_B)})}{k_A+k_B} - mk_A, \\
 \text{eq.4 in}~\eqref{eq:optbf2}~~   & b = \frac{-(d-b)k_B}{k_A+k_B}, \\
  \text{eq.1 and 3 in}~\eqref{eq:optbf2}~~  & x_2(t): = e^{a-c-m(k_A+k_B)} = \frac{k_A}{\sqrt{(k_B+k_An_B t^2/n_A)(k_A+k_B)}-k_B},\\
 \text{eq.2 in}~\eqref{eq:optbf2}~~  & x_1(t):=e^{d-b+m(k_A+k_B)} = \frac{k_B}{(N\lambda_Z)^{-1}\sqrt{( k_An_B+k_Bn_At^{-2})(k_A + k_B)} - k_A},\\
  \text{eq.5 in}~\eqref{eq:optbf2}~~  & m = \frac{\lambda_Z}{\lambda_b}\frac{n_A-n_Bt}{\sqrt{( k_Bn_A +k_An_B t^2) (k_A+k_B)}}
  \end{aligned}
\end{equation}
\\
\textbf{Existence of a solution to~\eqref{eq:bf2ops}: }Similarly, we consider the following single-variable nonlinear equation for $t$,
\begin{equation}
  \begin{aligned}\label{EQ:Retwb}
  t & = \frac{-k_B(d-b)}{g_2(x_2(t)) - k_A(k_A+k_B)m(t)} = : \frac{f_1(t)}{f_2(t)} = f(t) 
  \end{aligned}
\end{equation}
where $x_1(t)$ and $x_2(t)$ are defined in~\eqref{eq:bf2ops}, and $g_2$ is in~\eqref{def:g12}. Equivalently we can also write $f_1(t) = -k_B\left[\log{x_1(t)} - (k_A+k_B)m(t)\right]$.

The idea of the proof is similar to the case (a): by using the monotonicity of $f_1$ and $f_2$, and also Lemma~\ref{lem:func}.
We denote the domain of $f_1(t)$ as $\mathcal{D}$, i.e.,
\begin{align*}
\mathcal{D} & = \left\{t>0~|~(N\lambda_Z)^{-1}\sqrt{(k_Bn_At^{-2}+k_An_B)(k_A+k_B)}>k_A\right\} \\
& = \begin{cases}
t \in\RR, & N\lambda_Z  \leq \sqrt{n_B(k_A+k_B)/k_A}, \\
t < \sqrt{\frac{k_Bn_A(k_A+k_B)}{(N\lambda_Z)^2 k_A^2 - k_An_B(k_A+k_B)}}, & N\lambda_Z  > \sqrt{n_B(k_A+k_B)/k_A}, 
\end{cases}
\end{align*}
and that of $f_2(t)$ is $\RR.$ In their domains, $x_1(t)$ and $x_2(t)$ are increasing and decreasing respectively, which implies $f_1(t)$ and $f_2(t)$ are strictly decreasing and increasing in $t$ respectively. It is straightforward to verify:
\[
\lim_{t \to 0^+}f_1(t) = +\infty, ~~\lim_{t\to\infty}f_2(t) = +\infty,
\]
where $m(t)$ stays bounded for any $t$,  $\lim_{t\rightarrow 0^+}f_1(t) = \infty$ due to $\lim_{t\rightarrow 0^+}x_1(t) = 0^+$, and $\lim_{t\rightarrow\infty}f_2(t) = \infty$ due to $\lim_{t\rightarrow \infty}x_2(t) \to 0^+$ and $\lim_{x\rightarrow 0^+}g_2(x) = \infty$.
It suffices to show that $f_1(t)$ and $f_2(t)$ share a common positive part. The following argument divides into two subcases whether $f_2(t)$ has a root or not. 

Suppose $\eta(\lambda_Z)$ in~\eqref{def:eta} is positive ($f_2(t)$ has no root), then  $f_1$ and $f_2$ share a common positive part since as $t\rightarrow 0^+$, $f_1$ goes to $\infty$ and $f_2$ stays positive.
For $\lambda_Z$ with $\eta(\lambda_Z) < 0$, $f_2$ has a unique zero $t^*(\lambda_Z)$. 
To ensure $f_1$ and $f_2$ share a common positive part, it suffices to have
\[
f_1(t^*(\lambda_Z)) = -k_B\log x_1(t^*(\lambda_Z))+(k_A+k_B)k_Bm(t^*(\lambda_Z)) > 0,
\]
which is equivalent to
\begin{align*}
& x_1(t^*(\lambda_Z)) = \frac{k_B}{(N\lambda_Z)^{-1}\sqrt{(k_Bn_At^*(\lambda_Z)^{-2}+k_An_B)(k_A+k_B)}-k_A} < e^{(k_A+k_B)m(t^*(\lambda_Z))} \\
& \Longleftrightarrow k_B e^{-(k_A+k_B)m(t^*(\lambda_Z))} < \frac{1}{N\lambda_Z}\sqrt{\left(\frac{k_Bn_A}{t^*(\lambda_Z)^2}+k_An_B\right)(k_A+k_B)} - k_A.
\end{align*}
This finishes the proof, as~\eqref{cond:caseb} guarantees $f_1(t^*(\lambda_Z)) > 0$ based on the argument above.

Note that Lemma~\ref{lem:supp} implies that the inequality above holds for any $\lambda_Z = \sqrt{n_B}/N+\delta$ with $\delta<\eps$. As a result, for any $\lambda_Z$ close to $\sqrt{n_B}/N$, the case (b) in Proposition~\ref{prop:bias} holds.
\end{proof}

\begin{proof}[\bf Proof of Proposition~\ref{prop:bias}(c)] We consider $\sqrt{n_B} < N\lambda_Z < \sqrt{n_A}$ and $\lambda_Z$ satisfies~\eqref{cond:casec}.  
\textbf{Solution Structure:} In this case, we propose the structure:
\begin{equation}\label{eq:struc3}
  \begin{aligned}
 \bar{ \BZ} = \mathcal{B}(a,0,0,0)
   & = \begin{bmatrix}
  a(k_A\I_{k_A} -\BJ_{k_A\times k_A})   & 0 \\
  0& 0 \end{bmatrix}, \\   \bar{\BR} =\mathcal{B}(a_R,b_R,c_R,d_R) &= \begin{bmatrix}
    0 & 0 \\
    0 & \alpha(\BI_{k_B}-\BJ_{k_B}/k_B)
    \end{bmatrix}  + 
\frac{k_Ak_B\tau}{\sqrt{(k_Bn_A+k_An_Bt^2)(k_A+k_B)}} 
\bs\bs^{\top}
\begin{bmatrix}
\I_{k_A} & 0 \\
0 & t\I_{k_B}
\end{bmatrix} 
\end{aligned}
\end{equation}
where
\begin{equation}\label{eq:ar1}
\begin{aligned}
  & k_Aa_R + k_B c_R = 0,~~k_Ab_R + k_B d_R = \alpha, \\
  & b_R = \frac{\tau t}{\sqrt{(k_Bn_A+k_An_Bt^2)(k_A+k_B)}},~~c_R =  \frac{\tau}{\sqrt{(k_Bn_A+k_An_Bt^2)(k_A+k_B)}},
\end{aligned}
\end{equation}
which satisfies $\bar{\BR}\BD\bar{\BZ}^{\top} = \bar{\BZ}^{\top}\bar{\BR}=0$ and $\|\bar{\BR}\BD^{1/2}\|\leq 1$ holds as long as 
\[
\max\{\sqrt{n_B}|\alpha|, \tau\} \leq 1.
\]
\\
\textbf{Optimality system: }Based on structure~\eqref{eq:struc3}, we have~\eqref{eq:az} reduce to 
\begin{equation}\label{eq:az1}
\begin{aligned}
k_Aa_Z + k_Bc_Z & = n^{-1/2}_A\sign(a), ~~~k_Ab_Z + k_Bd_Z = 0, ~~b_Z =c_Z = 0, 
\end{aligned}
\end{equation}
and~\eqref{eq:ap} becomes
\begin{equation}\label{eq:ap1}
\begin{aligned}
& a_P = \frac{e^{-a+mk_B}}{k_Be^{-mk_A}+e^{-a+mk_B}(k_A-1 + e^{k_Aa } )},~~~b_P = \frac{e^{mk_B}}{k_Ae^{mk_B} + k_Be^{-mk_A}}, \\
& c_P = \frac{e^{-mk_A} }{k_Be^{-mk_A}+e^{-a+mk_B}(k_A-1 + e^{k_Aa } )},~~~d_P = \frac{e^{-mk_A}}{k_Ae^{mk_B} + k_Be^{-mk_A}}.
\end{aligned}
\end{equation}
Using the optimality condition $\I_K - \bar{\BP} = N\lambda_Z( [(\bar{\BZ}\BD\bar{\BZ}^{\top})^{\dagger}]^{1/2}\bar{\BZ} + \bar{\BR})$ leads to a normal equation similar to \eqref{EQ:OP:Original}:

\begin{equation*}
  \begin{aligned}
  k_Aa_P+k_Bc_P &= N\lambda_Z(k_Aa_Z+k_Bc_Z),~~k_Ab_P+k_Bd_P = 1 = N\lambda_Z\alpha,\\
 ~~b_P &=N\lambda_Z \left(b_Z + b_R\right)> 0,~~c_P = N\lambda_Z \left(c_Z + c_R \right)> 0.
  \end{aligned}
\end{equation*}
where $(a_Z,b_Z,c_Z,d_Z)$, $(a_R,b_R,c_R,d_R)$ and  $(a_P,b_P,c_P,d_P)$ satisfy~\eqref{eq:az1},~\eqref{eq:ar1} and~\eqref{eq:ap1} respectively.
By comparing the coefficients, we have
\begin{equation}\label{eq:caseopt_c}
\begin{aligned}
& \frac{k_Ae^{-a+mk_B}+k_Be^{-mk_A}}{k_Be^{-mk_A}+e^{-a+mk_B}(k_A-1 + e^{k_Aa } )} = \frac{N\lambda_Z}{\sqrt{n_A}}, \\
& \frac{e^{mk_B}}{k_Ae^{mk_B} + k_Be^{-mk_A}} =   \frac{N\lambda_Z\tau t}{{\sqrt{(k_Bn_A+k_An_Bt^2)(k_A+k_B)}}}, \\
& \frac{e^{-mk_A}}{k_Be^{-mk_A}+e^{-a+mk_B}(k_A-1 + e^{k_Aa } )} =  \frac{N\lambda_Z\tau}{{\sqrt{(k_Bn_A+k_An_Bt^2)(k_A+k_B)}}}, \\
& m = \frac{\lambda_Z\tau}{\lambda_b } \frac{n_A-n_Bt}{\sqrt{(k_Bn_A+k_An_Bt^2)(k_A+k_B)}},~~~\tau\leq 1,~~~N\lambda_Z \alpha = 1.
\end{aligned}
\end{equation}
Here $\sqrt{n_B}\alpha = \sqrt{n_B}/(N\lambda_Z)\leq 1$ with $\alpha = 1/N\lambda_Z$ if $\lambda_Z \geq (\sqrt{n_B}{N})$. Let $x(t,\tau) := e^{a-m(k_A+k_B)}$, and then  the optimality system~\eqref{eq:caseopt_c} becomes 
  \begin{equation}\label{eq:opsc3wbt}
    \begin{aligned}
 \text{eq.1 in}~\eqref{eq:caseopt_c} ~~ &  h(t,\lambda_Z,\lambda_b,\tau)  : = \log{\left[\left(\frac{\sqrt{n_A}}{N\lambda_Z}-1\right)\left(k_A+k_Bx(t,\tau)\right)+1 \right]} \\
    & \qquad - k_A\log{x(t,\tau)} - (k_A+k_B)k_Am(t,\lambda_Z,\lambda_b,\tau) = 0, \\
  \text{eq.1 and 3 in}~\eqref{eq:caseopt_c} ~~ &   x(t,\tau)  : = e^{a-(k_A+k_B)m} \\
  & \qquad =  \frac{k_A}{\tau^{-1}\sqrt{(k_B+t^2k_An_B/n_A)(k_A+k_B)}-k_B}, \\
   \text{eq.4 in}~\eqref{eq:caseopt_c} ~~ & m(t,\lambda_Z,\lambda_b,\tau)  := \frac{\lambda_Z\tau}{\lambda_b} \frac{n_A-n_Bt}{\sqrt{(k_Bn_A+k_An_Bt^2)(k_A+k_B)}},\\
  \text{eq.2 in}~\eqref{eq:caseopt_c}, ~\text{eq.2-3 in}~\eqref{eq:opsc3wbt}~~      & t  = \frac{\sqrt{n_A}(k_A/x(t,\tau)+k_B)}{N\lambda_Z(k_A+k_Be^{-(k_A+k_B)m(t,\lambda_Z,\tau)})}, \\
  \text{eq.4 in}~\eqref{eq:caseopt_c} ~~ & \tau \leq 1.
    \end{aligned}
  \end{equation}
In particular, if $\tau = 1$, $h(t,\lambda_Z,\lambda_b,1) = f_2(t,\lambda_Z,\lambda_b)$ holds where $f_2$ is given in~\eqref{def:f2proof}. When no confusion arises, we denote $m(t,\lambda_Z,\lambda_b,\tau)$ and $x(t,\tau)$ by $m(t)$ and $x(t)$ respectively. Now the goal is to prove the existence of solution set $(a,m,t,\tau)$ of the above system.\\
\\
\textbf{Existence of a solution to~\eqref{eq:opsc3wbt}:} The proof follows from two steps: (i) given $\lambda_Z$, $h(t,\lambda_Z,\tau)$ has a root $t(\lambda, \tau,\lambda_b)$ for any $\tau\in [\tau^*_{\lambda_Z,\lambda_b},1]$ where $\tau^*_{\lambda_Z,\lambda_b}$ is a number only depends on $\lambda_Z$ and $\lambda_b$ if 
\[
f_2(0,\lambda_Z,\lambda_b) = h(0,\lambda_Z,\lambda_b,1) < 0
\]
so that the first three equations in~\eqref{eq:opsc3wbt} satisfy;
(ii) we show that there exists a $\tau \leq 1$ such that the fourth equation in~\eqref{eq:opsc3wbt} also holds. The combination of steps (i) and (ii) is sufficient to prove the existence of the solution, we prove them respectively.\\
\\
\textbf{Proof for step (i):} For simplicity, we use $t(\lambda_Z,\tau)$ or $t$ to replace $t(\lambda_Z,\lambda_b,\tau).$ Observe that $x$ and $m$ are both determined by $t$ and $\tau$ according to eq $2$ and $3$ of~\eqref{eq:opsc3wbt}, so the key is to prove the existence of $t$ as the root of $h$.
Note that the domain of $h(t,\lambda_Z,\tau)$ is $\RR$ for any fixed $0<\tau\leq 1$ and $\sqrt{n_B}\leq N\lambda_Z\leq \sqrt{n_A}$, and $h$ is strictly increasing in $t$ with $\lim_{t\rightarrow\infty} h(t,\lambda_Z,\tau) = \infty$ and
\[
h(0,\lambda_Z,\lambda_{b},\tau) = \log{\left[\left(\frac{\sqrt{n_A}}{N\lambda_Z}-1\right)\left(k_A+k_Bx(0)\right)+1 \right]} - k_A\log{x(0)} - (k_A+k_B)k_Am(0)
\]
where
\[
m(0) =  \frac{\lambda_Z\tau}{\lambda_b}  \sqrt{ \frac{n_A}{k_B(k_A+k_B)} },~~~x(0) =  \frac{k_A}{\tau^{-1}\sqrt{k_B(k_A+k_B)}-k_B}.
\]
This implies $h(t,\lambda_Z,\lambda_{b},\tau)$ has a unique solution in $t$ for a given triple of $(\lambda_Z,\lambda_{b},\tau)$ if and only if $h(0,\lambda_Z,\lambda_{b},\tau) < 0$.

We can also see that $h(0,\lambda_Z,\lambda_{b},\tau)$ is decreasing in $\lambda_Z$ and $\tau$. Note that
\[
\lim_{\tau\rightarrow 0^+}h(0,\lambda_Z,\lambda_b,\tau) = +\infty.
\]
For $h(0,\lambda_Z,\lambda_b,1)< 0$, we define $\tau^*_{\lambda_Z,\lambda_b}\in [0,1]$ as
\[
h(0,\lambda_Z,\lambda_b,\tau^*_{\lambda_Z,\lambda_b}) = 0
\]
where $\tau^*_{\lambda_Z,\lambda_b}$ is the unique zero of $h(0,\lambda_Z,\lambda_b,\tau)$ in $\tau$ as $h(0,\lambda_Z,\lambda_b,\tau)$ is  continuous in $\tau$ for any $\lambda_Z.$

Now we define $t(\lambda_Z,\tau)$ be the zero to $h(t,\lambda_Z,\lambda_{b},\tau)$, i.e., 
\[
h(t(\lambda_Z,\tau),\lambda_Z,\lambda_{b},\tau) = 0,~~~\forall  \tau\in [\tau^*_{\lambda_Z,\lambda_b}, 1].
\]
In particular, $t(\lambda_Z,\tau^*_{\lambda_Z,\lambda_b}) = 0$ holds, so the existence of $x(t(\lambda_Z,\tau))$ and $m(t(\lambda_Z,\tau))$ follow.\\
\\
\textbf{Proof for step (ii): }It suffices to find a $\tau$ such that the fourth equation holds.
Let
\[
L(\lambda_Z,\tau) = t(\lambda_Z,\tau) -  \frac{\sqrt{n_A}(k_A/x(t(\lambda_Z,\tau),\tau)+k_B)}{N\lambda_Z(k_A+k_Be^{-(k_A+k_B)m(t(\lambda_Z,\tau),\lambda_Z,\tau)})}
\]
for any $\tau \in [\tau^*_{\lambda_Z,\lambda_b},1]$ and $\lambda_Z$ with $f_2(0,\lambda_Z) \leq 0.$

To show $L(\lambda_Z,\tau)$ has a zero in $\tau$, we check the value of $L(\lambda_Z,\tau)$ at $\tau = \tau^*_{\lambda_Z,\lambda_b}$ and $\tau=1$. At $\tau = \tau^*_{\lambda_Z,\lambda_b}$, it holds $t(\lambda_Z,\tau^*_{\lambda_Z,\lambda_b}) = 0$ and 
\[
L(\lambda_Z,\tau^*_{\lambda_Z,\lambda_b})  = -\frac{\sqrt{n_A}(k_A/x(0,\tau^*_{\lambda_Z,\lambda_b}),\tau)+k_B)}{N\lambda_Z(k_A+k_Be^{-(k_A+k_B)m(0,\lambda_Z,\tau^*_{\lambda_Z,\lambda_b})})} < 0.
\]
At $\tau = 1$, we have $t(\lambda_Z,1) = t^*(\lambda_Z)$, $x(t^*(\lambda_Z),1) = x_2(t^*(\lambda_Z))$, $m(t^*(\lambda_Z),\lambda_Z,1) = m(t^*(\lambda_Z))$, and \[
\frac{k_A}{x(t^*(\lambda_Z),1)} + k_B = \sqrt{(k_B + k_An_Bt^*(\lambda_Z)^2/n_A)(k_A+k_B)}
\]
follows from the second equation in~\eqref{eq:opsc3wbt}.
Therefore, 
\begin{align*}
L(\lambda_Z,1) & =t^*(\lambda_Z) -  \frac{\sqrt{n_A}(k_A/x(t^*(\lambda_Z),1)+k_B)}{N\lambda_Z(k_A+k_Be^{-(k_A+k_B)m(t^*(\lambda_Z))})} \\
& = t^*(\lambda_Z) -  \frac{\sqrt{(k_Bn_A + k_An_Bt^*(\lambda_Z)^2)(k_A+k_B)}}{N\lambda_Z(k_A+k_Be^{-(k_A+k_B)m(t^*(\lambda_Z))})} > 0
\end{align*}
which is guaranteed by the condition~\eqref{cond:caseb},
\[
k_A+k_Be^{-(k_A+k_B)m(t^*(\lambda_Z))} >\frac{1}{N\lambda_Z}  \sqrt{\left(\frac{k_Bn_A}{t^*(\lambda_Z)^2}+k_An_B\right)(k_A+k_B)}.
\]
By continuity of $L(\lambda_Z,\tau)$ in $\tau$, there exists a choice of $\tau$ such that the fourth equation holds. Therefore, the condition~\eqref{cond:caseb} guarantees a solution to system~\eqref{eq:opsc3wbt}.

Note that Lemma~\ref{lem:supp} implies that the inequality above holds for any $\lambda_Z = \sqrt{n_A}/N-\delta$ with $\delta<\eps$. As a result, for any $\lambda_Z$ close to $\sqrt{n_A}/N$, the case (c) in Proposition~\ref{prop:bias} holds.
\end{proof}

\begin{proof}[\bf Proof of Proposition~\ref{prop:bias}(d)] We consider $N\lambda_Z \geq \sqrt{n_A}$. \\
\textbf{Solution structure:} In this case, we propose the solution structure:
\begin{equation}\label{eq:struc4}
 \bar{ \BZ} = 0, ~~~   \bar{\BR} =\mathcal{B}(a_R,b_R,c_R,d_R)
\end{equation}
Aagin, we directly have  $\bar{\BR}\BD\bar{\BZ}^{\top} = \bar{\BZ}^{\top}\bar{\BR}=0$ and we are left with verifying $\|\bar{\BR}\BD^{1/2}\|\leq 1$ through optimality condition.\\
\\
\textbf{Optimality system:} The optimality condition is straightforward by~\eqref{cond:1stred} when $\BZ=0$. When $\BZ=0$, we have
\begin{equation}\label{eq:case4opor}
  \begin{aligned}
    & N^{-1}(\BI_K - \bar{\BP}) = \lambda_Z \bar{\BR}, \\
    & N^{-1}(\I_K-\bar{\BP})\bn = \lambda_b \bb, \\
    & \|\bar{\BR}\BD^{1/2}\|\leq 1.
  \end{aligned}
\end{equation}
For ease of notation, we denote $u=e^{-mk_A},v=e^{-mk_B}$ and $w=v/u$. The system~\eqref{eq:case4opor} above reduces to
\begin{equation}\label{eq:opsc4wbt}
  \begin{aligned}
  \text{eq.1 and 3 in}~\eqref{eq:case4opor} ~~ & \|(\BI_K-\bar{\BP}) \BD^{1/2}\| \leq N\lambda_Z, \\
  \text{eq.2}~~\eqref{eq:case4opor} ~~ &  \frac{n_A-n_Bw}{k_B+k_Aw} = N\lambda_b \frac{\log{w}}{k_A+k_B},\\
  & m = (k_A+k_B)^{-1} \log w.
  \end{aligned}
\end{equation}
Now we proceed to find a solution to system~\eqref{eq:opsc4wbt}.\\
\\
\textbf{Existence of a solution to~\eqref{eq:opsc4wbt}: }We compute the SVD of $(\BI_K-\bar{\BP}) \BD^{1/2}$ directly,

  \begin{equation}
    \begin{aligned}
      (\BI_K - \bar{\BP})\BD^{1/2} &= \begin{bmatrix}
      \sqrt{n_A}(\BI_{k_A}-\BJ_{k_A}/k_A) & 0 \\
      0 & \sqrt{n_B}(\BI_{k_B}-\BJ_{k_B}/k_B) \end{bmatrix} +\\
      & \frac{1}{k_Bu+k_Av} \begin{bmatrix} 
      \frac{k_B\sqrt{n_A} u}{k_A} \BJ_{k_A} & -\sqrt{n_B}v \BJ_{k_A \times k_B}\\
      -\sqrt{n_A}u \BJ_{k_B \times k_A} & \frac{k_A\sqrt{n_B}}{k_B}v \BJ_{k_B}\end{bmatrix}
    \end{aligned}
  \end{equation}
Hence we have,
\begin{equation}
  \begin{aligned}
    &(\BI_K -\bar{\BP})\BD(\BI_K - \bar{\BP})^{\top} = \begin{bmatrix}
    n_A(\BI_{k_A}-\BJ_{k_A}/k_A) & 0 \\
    0 & n_B(\BI_{k_B}-\BJ_{k_B}/k_B) \end{bmatrix} +\\
    & \qquad \frac{1}{(k_Bu+k_Av)^2} \begin{bmatrix} 
    (\frac{k_B^2n_A}{k_A}u^2+k_Bn_Bv^2)\BJ_{k_A} & -(k_Bn_Au^2 + k_An_Bv^2) \BJ_{k_A \times k_B}\\
    -(k_Bn_Au^2 + k_An_Bv^2) \BJ_{k_B \times k_A} & (\frac{k_A^2n_B}{k_B}v^2+k_An_Au^2) \BJ_{k_B}\end{bmatrix}
  \end{aligned}
\end{equation}
Therefore, the singular values of $(\BI_K - \bar{\BP})\BD^{1/2}$ are $\sqrt{n_A}$ with multiplicity $k_A-1$, $\sqrt{n_B}$
with multiplicity $k_B-1$, $\sqrt{(k_A+k_B)(k_Bn_Au^2+k_An_Bv^2)}/(k_Bu+k_Av)$ with multiplicity $1$, and  $0$ with multiplicity $1$. It suffices to show the 
maximum singular value is given by $\sqrt{n_A}$. Hence eq.1 in~\eqref{eq:opsc4wbt} is satisfied when $\lambda_Z \geq \sqrt{n_A}/N$. We only need to look into the squared singular value
\[
\sigma(w): = \frac{(k_A+k_B)(k_Bn_Au^2+k_An_Bv^2)}{(k_Bu+k_Av)^2} = \frac{(k_A+k_B)(k_Bn_A+k_An_Bw^2)}{(k_B+k_Aw)^2}.
\]
The objective is then to show the existence of $w$ as the solution of eq.2 in~\eqref{eq:opsc4wbt} and $\sigma(w) \leq n_A$ satisfies for that $w$. The idea is to constrain the range of $w$ and bound it by the monotonicity of $\sigma(w)$.

We claim $w\in[1,n_A/n_B]$ by checking two ends of eq.2 in~\eqref{eq:opsc4wbt}. On the one hand, the LHS is greater than $0$ iff $w<n_A/n_B$ and the RHS is greater than $0$ iff $w>1$. On the other hand, on $w \in [1,n_A/n_B]$, the LHS (RHS) strictly decreases (increases) in $w$ and at $w=1\text{ }(n_A/n_B)$, we have LHS$>$RHS (LHS$<$RHS). These allow us to conclude that there exists a unique $w\in \left[1, n_A/n_B \right]$ that satisfies eq.2 of~\eqref{eq:opsc4wbt}.

Now we check the monotonicity of $\sigma(w)$ by computing its derivative:
\begin{align*}
\sigma'(w) & = 2(k_A+k_B)\frac{k_An_Bw(k_B+k_Aw) - k_A (k_Bn_A + k_An_Bw^2)}{ (k_B+k_Aw)^3 }\\
& = 2k_Ak_B(k_A+k_B) \frac{n_Bw-n_A}{(k_B+k_Aw)^3} \leq 0, \text{ for } w\in \left[1,\frac{n_A}{n_B} \right].
\end{align*}
Therefore when $w=1$, we obtain an upper bound of this singular value,
\[
\sigma(1) = \frac{k_Bn_A+k_An_B}{k_A+k_B} \leq n_A,
\]
which implies the largest singular value of $(\BI_K - \bar{\BP})\BD^{1/2}$ is no larger than $\sqrt{n_A}$. This verifies eq.1 of~\eqref{eq:opsc4wbt}.
\end{proof}

\begin{proof}[\bf Proof of Proposition~\ref{prop:bias}(e)] 
If $\lambda_b = \infty$, then~\eqref{cond:caseb} and~\eqref{cond:casec} are equivalent to $\xi(\lambda_Z,\infty) < 0$ and $\xi(\lambda_Z,\infty) > 0$ in~\eqref{def:xi2}. Lemma~\ref{lem:supp}(d) implies that $\xi(\lambda_Z,\infty)$ is increasing in $\lambda_Z$ with $\xi(\sqrt{n_B}/N,\infty) < 0$ and $\xi(\sqrt{n_A}/N,\infty) > 0$. Therefore, there exists a $\lambda^*$ which is the root to $\xi(\lambda_Z,\infty)$, such that~\eqref{cond:caseb} and~\eqref{cond:casec} are equivalent to $\lambda_Z<\lambda^*$ and $\lambda_Z>\lambda^*$ respectively.
\end{proof}

\subsection{Limiting case: Proof of Corollary~\ref{corollary:mc} and Theorem~\ref{thm:limit}}
In this section, we give some asymptotic characterization of $\bar{\BZ}$, when either $n_A$ or $n_B$, or both go to infinity. 

\begin{proof}[\bf Proof of Corollary~\ref{corollary:mc}]
From Proposition~\ref{prop:bias}, the minority collapse occurs when $N\lambda_Z \geq \sqrt{n_B}$. Plugging in $N=k_An_A+k_Bn_B = n_B(k_Ar +k_B)$ yields 
\[
n_B(k_Ar +k_B)\lambda_Z\geq \sqrt{n_B} \Longleftrightarrow r \geq \frac{1}{k_A}\left( \frac{1}{\sqrt{n_B}\lambda_Z} -k_B\right).
\]
\end{proof}

  \begin{proof}[\bf Proof of Theorem~\ref{thm:limit}]
Under the assumption $N\lambda_Z < \sqrt{n_B}$, the mean feature matrix $\bar{\BZ}$ falls in the case (a) of Theorem~\ref{thm:biasf}. 
The key is to show that the unique solution $t^*_N$ of~\eqref{EQ:t} (i.e., the reduced optimality condition for case (a)) converges to $1$ at the rate of $1/\log N$ as $N \to \infty$.  
  For simplicity, we define
  \[
  f_N(t) = \frac{g_1(x_1(t))+(k_A+k_B)k_Bm(t)}{g_2(x_2(t))-(k_A+k_B)k_Am(t)}
  \] 
  indexed by the total sample size. Let $h_N(t) = f_N(t) - t$ and $t_N$ be the zero of $h_N(t)$. From the previous analysis, we know $t_N$ is unique and $h_N(t)$ is monotonically decreasing on $I_{f_N}^+: = \{t:f_N(t)>0\}.$ Moreover, it holds $f_N'(t) < 0$ on $I_{f_N}^+$ and as a result, 
  \[
  h_N'(t) = f_N'(t) - 1 < -1.
  \]
  This implies that
  \[
  \left| h_N(t)-h_N(t') \right| \geq \left|t-t' \right| \quad \text{for any } t,t' \in I_{f_N}^+.
  \]
  From the following argument, it is straightforward to see that for sufficiently large $N$, $g_1(x_1(1))$ and $g_2(x_2(2))$ are both positive and dominate $m(1)$, which implies $1 \in I_{f_N}^+$. So we can obtain the following bound:
  \begin{align*}
  \left| t_N - 1 \right| & \leq \left| h_N(t_N) - h_N(1) \right| = |h_N(1)| \\
  & = \left| \frac{\log{\left[\left(\frac{\sqrt{n_B}}{\lambda}-1\right)\left(k_Ax_1(1)+k_B\right)+1 \right]}-k_B\log x_1(1) + (k_A+k_B)k_Bm(1)}{\log{\left[\left(\frac{\sqrt{n_A}}{\lambda}-1\right)\left(k_A+k_Bx_2(1)\right)+1 \right]}-k_A\log x_2(1) - (k_A+k_B)k_Am(1)} - 1 \right|, 
  \end{align*}
  where $x_1(t)$ and $x_2(t)$ are defined in~\eqref{def:x12case1}, satisfying
  \begin{equation}
  \begin{aligned}
x_1(1) = \frac{k_B}{\sqrt{(rk_B+k_A)(k_A+k_B)}-k_A},~~~x_2(1) = \frac{k_A}{\sqrt{(k_B+k_A/r)(k_A+k_B)}-k_B}.
  \end{aligned}
  \end{equation}
  It is easy to see $x_1(t)$ and $x_2(t)$ stay bounded for fixed $r$. As $N\rightarrow\infty$, we notice both $n_A$ and $n_B$ go to $\infty$, and also
  \[
  m(1) \lesssim \frac{\lambda}{\sqrt{N}\lambda_{b}} = o\left(\log(N)\right)
  \]
  follows from the assumption on the decay rate of $\lambda_b$.
  Thus sending $N\rightarrow\infty$ implies
  \[
  |t_N - 1| \lesssim \left|\frac{\log \sqrt{n_B}/\lambda + o(\log{N}) }{\log \sqrt{n_A}/\lambda+o(\log{N})}  - 1\right| = \frac{\log \sqrt{r}}{\log \sqrt{n_A}/\lambda} = O\left(\frac{1}{\log N}\right)
  \]
  for sufficiently large $N$.

  From~\eqref{EQ:Reduction:OPW}, we know that
  \begin{equation*}
    \begin{aligned}
  b_N &=  \log{\left[\left(\frac{\sqrt{n_B}}{\lambda}-1\right)\left(k_Ax_1(t_N)+k_B\right)+1 \right]}-k_B\log x_1(t_N) +o\left(\log{N}\right) \\
  &\geq \log \frac{\sqrt{n_B}}{\lambda} - k_B\log x_1(t_N) + o\left(\log{N} \right).
    \end{aligned}
  \end{equation*}
  Therefore, $b_N$ is at least $\log (\sqrt{n_B}/\lambda)$ and similarly $c_N \geq \log (\sqrt{n_A}/\lambda).$ The uniform boundedness of $x_1$ and $x_2$ in~\eqref{def:x12case1} implies that $a_N$ and $c_N$, and $b_N$ and $d_N$ grow at the same rate, i.e.,
  \[
  \lim_{N\rightarrow\infty}b_N =O(\log N) = \infty,~~\lim_{N\rightarrow\infty}c_N =O(\log N) = \infty
  \]
  and 
  \[
  \lim_{N\rightarrow\infty} t_N = \lim_{N\rightarrow\infty} \frac{b_N}{c_N} = 1,~~\lim_{N\rightarrow\infty} \frac{a_N}{c_N} = 1,~~\lim_{N\rightarrow\infty} \frac{d_N}{b_N} = 1.
  \]
  In other words, as $N\rightarrow\infty$, we have
  \[
  \lim_{N\rightarrow\infty}\frac{1}{b_N}\bar{\BZ} = (k_A+k_B)\I_{k_A+k_B} - \BJ_{k_A+k_B}
  \]
  which implies the column normalized $\bar{\BZ}$ in this limit should converge to the ETF, so do $\bar{\BH}$ and $\BW$.
\end{proof}

\end{document}